%% file: main.tex
\documentclass[10pt,twocolumn,letterpaper]{article}

\usepackage{cvpr}              % To produce the CAMERA-READY version
\usepackage{multirow}

\input{preamble}

\input{math_commands.tex}

\definecolor{cvprblue}{rgb}{0.21,0.49,0.74}
\usepackage[pagebackref,breaklinks,colorlinks,citecolor=cvprblue]{hyperref}
\usepackage[ruled, lined, longend, linesnumbered]{algorithm2e}
\usepackage{tikz}
\usepackage{todonotes}
\def\paperID{*****} % *** Enter the Paper ID here
\def\confName{CVPR}
\def\confYear{2026}
\newcommand\independent{\protect\mathpalette{\protect\independenT}{\perp}}
\def\independenT#1#2{\mathrel{\rlap{$#1#2$}\mkern2mu{#1#2}}}

\title{Computation and Communication Efficient Federated Unlearning \\
       via On-server Gradient Conflict Mitigation and Expression}
       
\author{Minh-Duong Nguyen $^1$, Senura Hansaja $^2$, Le-Tuan Nguyen $^1$, Quoc-Viet Pham $^3$, \\
Ken-Tye Yong $^2$, Nguyen H. Tran $^2$, Dung D. Le $^1$ \\
$^1$ {\color{orange}VinUniversity}, $^2$ {\color{orange}University of Sydney}, $^3$ {\color{orange}Trinity College Dublin}\\
{\tt\small \{duong.nm2, tuan.nl, dung.ld\}@vinuni.edu.vn, viet.pham@tcd.ie} \\
{\tt\small \{senu.wanasekara, nguyen.tran, ken.yong\}@sydney.edu.au}
}

\begin{document}
\maketitle

\input{sec/0_abstract}    
\input{sec/1_intro}

\input{sec/2_relatedworks}

\input{sec/3_methods}
% \input{sec/4_theoretical}
% \input{sec/5_settings}
\input{sec/6_experiments}

\input{sec/7_conclusion}
% \clearpage
{
    \small
    \bibliographystyle{ieeenat_fullname}
    \bibliography{main}
}


\clearpage
\appendix
\input{sec/b2_appendix}

\input{sec/b1_appendix}
\input{sec/b3_appendix}

\input{sec/b4_appendix}
\end{document}

%% file: preamble.tex
\usepackage[dvipsnames]{xcolor}
\usepackage{soul}
\usepackage{comment}
\usepackage{bbm}
\usepackage{multirow}

\newtheorem{definition}{Definition}
\newtheorem{theorem}{Theorem}

\newtheorem{assumption}{Assumption}

%% file: math_commands.tex
\usepackage{amsmath,amsfonts,bm}

\def\eqref#1{Eq.~(\ref{#1})}
\def\1{\bm{1}}

\def\rvg{{\mathbf{g}}}

\DeclareMathAlphabet{\mathsfit}{\encodingdefault}{\sfdefault}{m}{sl}
\SetMathAlphabet{\mathsfit}{bold}{\encodingdefault}{\sfdefault}{bx}{n}

\def\gD{{\mathcal{D}}}
\def\gE{{\mathcal{E}}}
\def\gF{{\mathcal{F}}}

\def\gL{{\mathcal{L}}}

\def\gR{{\mathcal{R}}}
\def\gS{{\mathcal{S}}}

\def\gU{{\mathcal{U}}}
\def\gV{{\mathcal{V}}}

\def\gX{{\mathcal{X}}}
\def\gY{{\mathcal{Y}}}
\def\gZ{{\mathcal{Z}}}

\newcommand{\var}{\mathrm{Var}}
\newcommand{\E}{\mathbb{E}}

\DeclareMathOperator*{\argmax}{arg\,max}
\DeclareMathOperator*{\argmin}{arg\,min}

%% file: sec/0_abstract.tex
\begin{abstract}
Federated Unlearning (FUL) aims to remove specific participants' data contributions from a trained Federated Learning model, thereby ensuring data privacy and compliance with regulatory requirements. Despite its potential, progress in FUL has been limited due to several challenges, including the cross-client knowledge inaccessibility and high computational and communication costs.
To overcome these challenges, we propose Federated On-server Unlearning (FOUL), a novel framework that comprises two key stages. The learning-to-unlearn stage serves as a preparatory learning phase, during which the model identifies and encodes the key features associated with the forget clients. This stage is communication-efficient and establishes the basis for the subsequent unlearning process.
Subsequently, on-server knowledge aggregation phase aims to perform the unlearning process at the server without requiring access to client data, thereby preserving both efficiency and privacy.
We introduce a new data setting for FUL, which enables a more transparent and rigorous evaluation of unlearning. To highlight the effectiveness of our approach, we propose a novel evaluation metric termed time-to-forget, which measures how quickly the model achieves optimal unlearning performance.
Extensive experiments conducted on three datasets under various unlearning scenarios demonstrate that FOUL outperforms the Retraining in FUL. 
% Unlike Retraining, which discards all knowledge from the forget clients, FOUL effectively leverages and refines this knowledge to enhance overall performance. 
Moreover, FOUL achieves competitive or superior results with significantly reduced time-to-forget, while maintaining low communication and computation costs. Reproducible code is available \textcolor{blue}{\href{https://github.com/skydvn/foul}{here}}.
\end{abstract}
% In the learning-to-unlearn phase, the feature extractor is disentangled into two sub-networks: one that encodes domain-invariant representations and another that captures domain-variant representations. We argue that by focusing the unlearning process solely on the domain-variant sub-network, FOUL effectively preserves global, task-relevant knowledge across all clients while efficiently removing the distinct features associated with the forget set.
% To 

%% file: sec/1_intro.tex
\section{Introduction}
\label{sec:intro}
Federated Learning (FL) has enabled the collaborative training of machine learning models across decentralized data sources or clients, facilitating the development of more accurate and robust models. Local clients benefit by getting an aggregated and more powerful personalized model without sharing their private data. 
With the recent introduction of data protection laws such as the General Data Protection Regulation (GDPR) \cite{regulation2018general} in the European Union, it has become essential to provide clients with the right to erasure regarding their own private data, even if that data has already been used to train the global model in federated learning.
% In the context of centralized, the right to erasure has led to recent research on machine unlearning \cite{di2025adversarial}.
% However, this collaborative approach also raises concerns about the integrity of the model \cite{} when requested to unlearn data from certain clients. In FL, models are trained on data from multiple clients, and the global model may inadvertently memorize information from individual data sources. This poses significant challenges when a client requests to remove their contribution from the global model due to contractual, legal compliance or privacy reasons. The global model may retain information about the client. 
Federated Unlearning (FUL) \cite{2023-FUL-AFU} aims to address this challenge by developing methods for removing specific information from globally trained models. FUL plays a critical role in supporting the \emph{right to be forgotten} paradigm, wherein the removal of a client's data from the model may be legally or ethically mandated. 

In the field of FUL, two primary categories of methods have emerged: retraining-based methods \cite{khalil2025not}, approximate unlearning methods \cite{2022-FUL-CDP, 2023-FUL-NFT} (see Appendix~\ref{sec:related-works} for more details).
Retraining-based methods involve retraining the entire FL system from scratch using a dataset that excludes the forget set, thereby ensuring that the final model no longer reflects the influence of the data to be forgotten. While these methods often achieve state-of-the-art unlearning performance, retraining approaches are computationally intensive. Consequently, such approaches are impractical and unscalable for real-world FL applications due to substantial computational and communication overhead.

To mitigate the computational burden of retraining-based unlearning, recent studies have shifted toward approximate unlearning methods~\cite{2022-FUL-CDP, 2023-FUL-NFT, 2024-FUL-HumanCentric}. These approaches typically focus on local unlearning or selective client aggregation, which are particularly effective in federated feature unlearning settings, where both the retain and forget subsets are available at the client side. In such cases, unlearning can be performed locally on clients and the updated models are aggregated in the same manner as vanilla FL~\cite{gu2024ferrari}. 

However, when the objective shifts to client-wise unlearning, where the goal is to erase the entire contribution of specific clients, the problem becomes substantially more challenging. The primary difficulty arises because the forget clients lack access to the retain dataset. To be more specific, current unlearning formulations are typically feasible only when both retain and forget data are available, as the unlearning objective is often expressed as $\gL_{\text{unlearn}} = \gL_{\text{retain}}(\theta; \gD_r) - \gL_{\text{forget}}(\theta; \gD_f)$.
Consequently, existing client-wise FUL methods rely solely on the retain set $\gD_r$, while ignoring the forgotten set $\gD_f$.
\begin{figure}[!t]
\centering
\includegraphics[width = 0.9\linewidth]{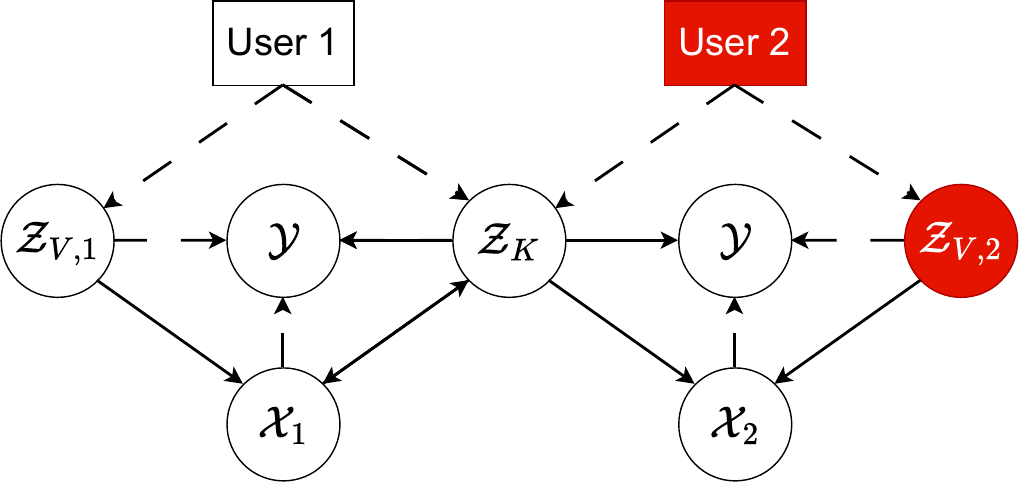}
% \caption{We solve Equation~\ref{eq:general-unlearn} by initializing the parameter vector $\theta$ with $\theta_{\text{learn}}$ and then applying a standard iterative optimization procedure. Here, $\theta^*_{\text{unlearn}}$ denotes the target unlearned model, while $\theta^*_{\text{learn}}$ represents the pre-trained model obtained during the learning stage. The term $\delta$ is a small positive threshold that constrains the divergence in knowledge between $\theta^*_{\text{unlearn}}$ and $\theta^*_{\text{learn}}$ over the remaining dataset.}
\caption{Extension of the causal structure model to the multi-client setting. Following \cite{lu2021invariant}, the causal factors $\gZ_K$ represent the domain-invariant characteristics. Consequently, when unlearning a client, removing $\gZ_K$ may reduce the amount of useful information necessary for predicting $\gY$. In contrast, unlearning the non-causal factors $\gZ_{V,2}$ primarily removes domain-specific information from client~2, while having a smaller impact on the information relevant to the classification of $\gY$. Dashed lines denote edges that may vary across clients or be absent in certain scenarios, solid lines indicate edges that remain invariant across all clients.}
\label{fig:causalremoval}
\end{figure}
In this work, we propose Federated On-server UnLearning (FOUL), a novel algorithm designed for client-wise FUL that efficiently and faithfully removes targeted knowledge without compromising the knowledge of retain clients. FOUL achieves this while significantly reducing the computational overhead of the unlearning process.
\noindent\tikz[baseline=(char.base)] \node[draw, circle, fill=black, text=white, inner sep=0.1pt](char){1}; To mitigate the high cost of unlearning, we introduce the principle of learning to unlearn (L2U) during the learning stage. Specifically, FOUL disentangles the feature extractor into two sub-networks: a causal featurizer, which captures domain-invariant representations, and a non-causal featurizer, which encodes domain-specific information. We demonstrate that unlearning only the non-causal features effectively removes domain-specific (i.e., forget-set-related) knowledge while preserving generalizable, domain-invariant information crucial to the retain set (see Fig.~\ref{fig:causalremoval}).
\noindent\tikz[baseline=(char.base)] \node[draw, circle, fill=black, text=white, inner sep=0.1pt](char){2}; To enable the joint utilization of knowledge from both retain and forget client sets, FOUL enforces aggregated gradients to remain aligned with the retain clients while intentionally conflicting with the forget clients, ensuring effective unlearning. Unlike prior methods such as~\cite{2024-FUL-USP}, FOUL requires no on-server data access, thereby preserving both communication efficiency and data privacy under the FL paradigm.
Our key contributions are summarized as follows:
\begin{enumerate}
    \item We propose FOUL, an efficient FUL algorithm which consists of two stages, L2U via model disentanglement, and on-server unlearning performed on the domain-specific disentangled model via gradient matching.
    \item We introduce an approach for leveraging domain generalization datasets (e.g., PACS, VLCS, OfficeHome) to benchmark FUL performance. Additionally, we propose a new metric, time-to-forget, which quantifies how quickly a model achieves optimal unlearning.
    \item We experimentally demonstrate the effectiveness of the proposed method across various client-wise unlearning settings, comparing it against $9$ FUL methods.
\end{enumerate}

%% file: sec/2_relatedworks.tex
\section{Preliminaries}
\label{sec:related}
\paragraph{Notations:} We abuse $|\cdot|$ as a number of elements in the subset, $\independent$ as independent relationship between two subsets. 
Suppose that there is a user set $\gU$. Let $\gD = \{(x_i, y_i)\}^{| \gD |}_{i=1}$, where $x_i$ and $y_i$ represent the input and its corresponding output, respectively. In our work, we refer $\gZ_K, \gZ_V$ as the causal and non-causal representation set, respectively. Due to the mini-batch learning, we abuse $\gZ_K = \{z^{(i)}_K \vert i\in B\}$ as the causal representation set within the batch with size $B$. The non-causal representation set $\gZ_V$ applies the same rule. 
% \hl{isn't this definition ambiguous}.

\subsection{Federated Unlearning} 
\begin{definition}[Forget user set]
    Given a training user set $\gU$, the forget user set is a subset of $\gU$ whose data is intended to be removed or unlearned, denoted by $\gU_{\gF}$.
\end{definition}

\begin{definition}[Retain user set]
    A \emph{retain user set} is the set of users which required to be memorized in the FL system, denoted as $\gU_\gR$, including all other data that is not targeted for removal and is crucial for maintaining the general knowledge and functionality of the model, i.e., $\gU_{\gR} = \gU \backslash \gU_{\gF}$.
\end{definition}
Given the two above definitions, FUL consists of two primary processes: 1) learning stage, and 2) unlearning stage. 

\begin{definition}[Learning Stage]
    The goal of \emph{learning stage} is to collaboratively learn a machine learning model $\theta$ over the dataset $\gD\triangleq \bigcup_{u\in \gU}\gD_u$: 
    \begin{align}
        {\theta}^*_\textrm{learn} = \argmin_{\theta} \frac{1}{| \gU|} \sum_{u\in\gU}\gL(\theta; \gD_u),
    \end{align}
    where $\gL(\theta; \gD_u)=\E_{(x,y)\sim \gD_u}[\ell(\theta; (x,y))]$ is the empirical loss of user $u$, and during \emph{learning stage}, each user minimizes their empirical risk $\gL(\theta; \gD_u)$, ${\theta}^*_\textrm{learn}$ is the final model trained by the \emph{learning stage}.
\end{definition}
\emph{The unlearning stage} is conducted following the learning stage, with the objective of obtaining an unlearned FL model, denoted as $\theta_\textrm{unlearn}$. 
% $\theta^*_\textrm{learn}$. 
\begin{definition}[Unlearning Stage]
     The goal of the \emph{unlearning stage} is to achieve performance comparable to a model trained exclusively on the remaining dataset, while ensuring that knowledge about the forget dataset is forgotten.
    \begin{align}
    \label{eq:general-unlearn}    
        &{\theta}^*_\textrm{unlearn} 
        = \argmax_{\theta} \frac{1}{ |\gU_\gF|} \sum_{v\in\gU_{\gF}}\gL(\theta; \gD_v), \\
        & \textrm{ s.t. } 
          \frac{1}{|\gU_\gR|}
          \sum_{u\in \gU_{\gR}}\Big[\gL({\theta}^*_\textrm{unlearn}; \gD_u) - \gL({\theta}^*_\textrm{learn}; \gD_u)\Big] \leq \delta, \notag
    \end{align}
    where we solve ~\ref{eq:general-unlearn} by initializing $\theta$ as $\theta_{learn}$ before using standard iterative optimization technique.
    $\theta_\textrm{unlearn}$ is the target unlearned model, and $\theta^*_\textrm{learn}$ is the pre-trained model of $\theta$ defined during the learning stage. $\delta$ is an arbitrarily small threshold which constraint the knowledge divergence between ${\theta}^*_\textrm{unlearn}$ and ${\theta}^*_\textrm{learn}$ for remaining set.
\end{definition}

% \begin{definition}[Unlearning Stage]
%     Let $\theta^*$ denote the global model produced by the federated learning pipeline and \hl{$\theta_u$ denote the global model produced by the federated unlearning pipeline}. For a given threshold $\epsilon \geq 0$, the federated unlearning problem is defined as: there exists an unlearning algorithm $\gA$ that makes the following equation hold: 
%     \begin{align}
%         \Vert T(\theta^*) - T_{\gA}(\theta^u) \Vert \leq \epsilon,
%     \end{align}
%     where $T(\cdot)$ denotes the distribution of models learned using unlearning algorithm $\gA$ and $\epsilon$ is a small positive number. $\gA$ achieves unlearning when these two distributions are $\epsilon$-identical. In other words, in our work, this evidence (i.e., the output of unlearning algorithm $\gA$) takes the form of a training algorithm, which if implemented correctly, guarantees that the parameter distributions of $\theta^*$ and $\theta_u$ are $\epsilon$-identical.
% \end{definition}

%% file: sec/3_methods.tex
\section{Methods}
In this section, we present the FOUL algorithm, which operates through a two-stage unlearning process: 
(1) Learning to unlearn, where the model is prepared during the initial learning stage to facilitate efficient future unlearning, and (2) On-server gradient matching, where unlearning is executed by manipulating gradients directly on the server. 

\subsection{Learning to unlearn}
Existing FUL approaches predominantly focus on retraining or adjusting the model from scratch \cite{gu2024ferrari}. However, this strategy incurs substantial computational overhead on individual clients and increases communication costs due to frequent model transmission between clients and the server.
In practice, not all model parameters are equally influenced by the data that must be forgotten \cite{khalil2025not, zhong2025unlearning}. Many components of the model capture domain-invariant knowledge that remains relevant to both the retained and forgotten data, and therefore do not require modification during the unlearning process.
To this extent, recent studies attempt to identify and unlearn only a subnetwork of the model by analyzing gradient conflicts~\cite{khalil2025not, zhong2025unlearning}. However, because the model changes across unlearning rounds, the subnetwork identification must be repeatedly performed, leading to significant additional computation overhead.
To address this limitation, we propose L2U, which disentangles the model into domain-invariant and domain-specific subnetworks during training. By isolating only the domain-specific component for unlearning, L2U preserves model stability across unlearning rounds while reducing computation and communication overhead.
\subsubsection{Motivation} 
The core idea of L2U via disentanglement is to learn disentangled representations by imposing invariance constraints on two spaces $\gZ_K$ and $\gZ_V$. 
Our design is inspired by the causal-structural representation learning~\cite{2022-DG-CIRL}, which is defined as follows: 
% \begin{theorem}[Causal Invariance~\cite{liu2021learning}]
%     The causal representations are invariant across domains, while the non-causal representations constitute the primary source of domain-specific variations.
% \end{theorem}
% The causal invariance theorem provides a foundation for our unlearning objective. Specifically, in the unlearning prhocess, we aim to remove the knowledge associated with the forget client set while preserving the knowledge from the retain client set. Consequently, the goal of unlearning can be interpreted as reducing the information encoded in the non-causal (domain-specific) representations. In contrast, modifying the causal representations would lead to a degradation of the overall retained knowledge.
% \begin{assumption}[Causal Principle in Unlearning]
%     To preserve the knowledge of the retain set while removing that of the forget set, the primary objective of unlearning is to reduce the information encoded by the variation encoder only.
% \end{assumption}
\begin{definition}[Causal Structured Model \cite{2022-DG-CIRL}]
    The data $\gX$ composes of two distinct representations, i.e., causal factors $\gZ_K$ and non-causal factors $\gZ_V$, and unexplained noise variables $\gV_1, \gV_2$ that are jointly independent, i.e., $\gZ_K \independent \gZ_V \independent \gV_1 \independent \gV_2$. 
    \begin{itemize}
        \item The causal factors $\gZ_K$ should be separated from the non-causal factors $\gZ_V$, i.e., $\gZ_K \independent \gZ_V$. Thus, performing an intervention upon $\gZ_V$ does not make changes to $\gZ_K$. 
        \item The combination of causal and non-causal factors $\gZ_K,\gZ_V$ should be sufficient to the data $\gX$ reconstruction task, i.e., $\gX \triangleq f(\gZ_K,\gZ_V,\gV_1)$.
        \item The causal factors $\gZ_K$ should be causally sufficient for the classification task $\gX\rightarrow \gY$, i.e., $\gY \triangleq h(\gZ_K,\gV_2)$. 
    \end{itemize}
    Here, $f,h$ can be regarded as unknown structural functions.
\label{def:csm}
\end{definition}
Given that the noise variables are unexplainable, we restrict our focus to the causal and non-causal representations. To relate these representations to data features across different domains, we adopt the following definition.
\begin{definition}[Causal and non-causal representations \cite{2022-DG-CIRL}]
    Causal factors $\gZ_K$ have a direct causal influence on both $X$ and $Y$, making them domain-invariant. In contrast, $\gZ_V$ represents non-causal factors that only causally affect $X$ and typically contain domain-specific information. 
\label{def:domain-causal}
\end{definition}
Definition~\ref{def:csm} formalizes the feasibility of disentangling the model into two independent factors. The causal invariance theorem provides the theoretical foundation for our unlearning objective. In the unlearning process, our goal is to remove the knowledge associated with the forget client set while preserving that of the retain client set. Therefore, unlearning can be viewed as reducing the information encoded in the non-causal (domain-specific) representations, since modifying the causal representations would inevitably impair the retained knowledge.
\begin{definition}[Causal Principle in Unlearning]
    To preserve the knowledge of the retain set while removing that of the forget set, the main objective of unlearning is to reduce the information encoded by the non-causal featurizer.
\label{def:causal-unlearning}
\end{definition}
Motivated by this principle, L2U disentangles the feature extractor into a causal featurizer $\theta_K$ and a non-causal featurizer $\theta_V$.
During the subsequent unlearning phase, updates are applied exclusively to the non-causal featurizer, while the causal featurizer remains frozen. This strategy enables FOUL to efficiently remove domain-specific knowledge from the forget set without compromising the generalizable representations learned from the retain set. 

\begin{figure}[!h]
\centering
\includegraphics[width = \linewidth]{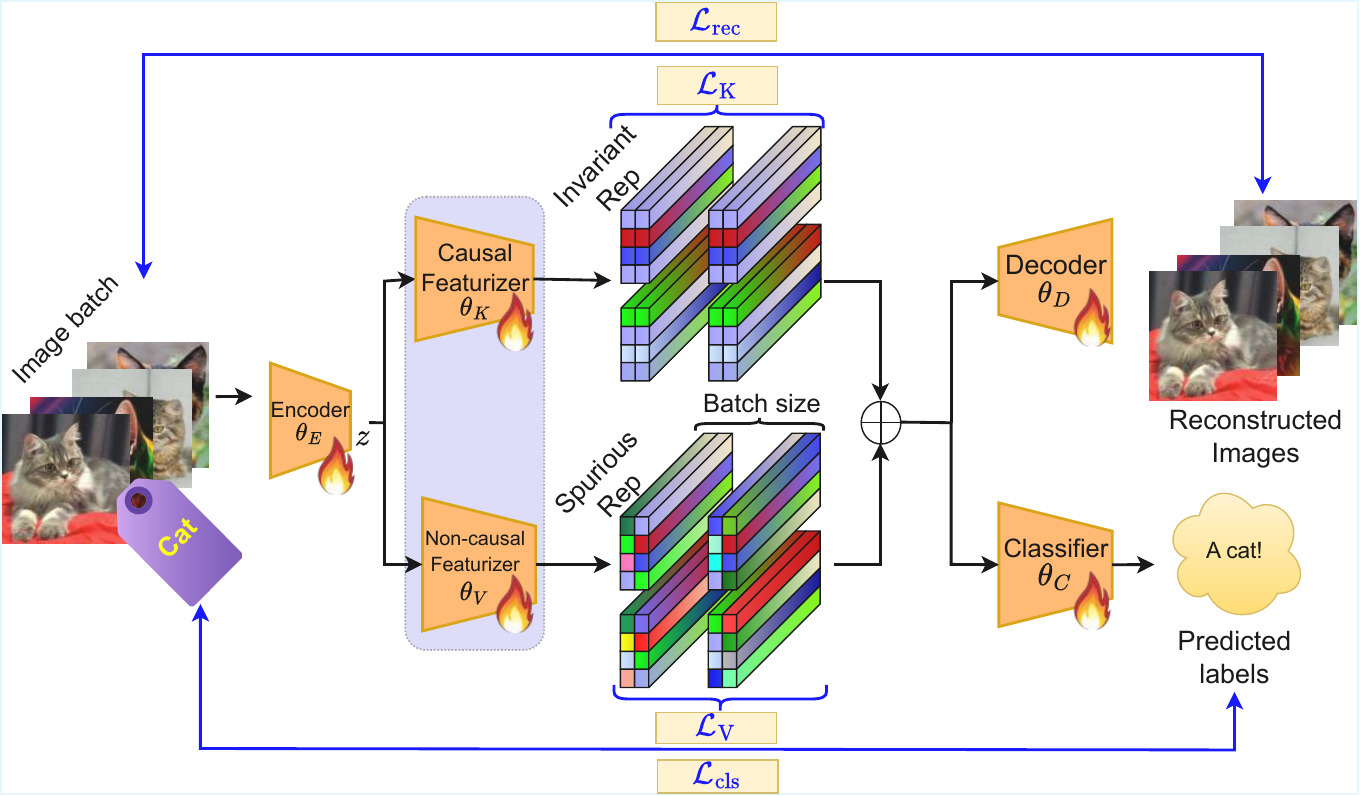}
\caption{The learning on a single local client in the L2U stage. Featurizers are trained to disentangle knowledge into invariant and variant representations.}
\label{fig:FOUL1}
\end{figure}

\subsubsection{Learning for Representation Featurizer}
\paragraph{General learning framework.} To achieve effective disentanglement learning in L2U, it is essential that the learned representations exhibit both causal and non-causal characteristics. These two components must not only capture distinct aspects of the data but also jointly contribute to form meaningful and interpretable representations. Moreover, the causal representations are required to preserve the underlying causal structure that remains invariant across domains. 
% Consequently, we formulate the following joint loss function on local clients as follows:
% \begin{align}
% \gL_\textrm{L2U} = \alpha_\textrm{V}\mathcal{L}_{\textrm{V}}
% + \alpha_\textrm{K}\mathcal{L}_{\textrm{K}}
% + \alpha_\textrm{gtc}\mathcal{L}_{\textrm{gtc}}
% + \gL_\textrm{rec},
% \label{eq:joint-FOUL}
% \end{align}
% {\color{blue}
% where $\mathcal{L}_{\textrm{V}}, \mathcal{L}_{\textrm{K}}, \mathcal{L}_{\textrm{gtc}}, \gL_\textrm{rec}$ are the non-causal, causal, classification, and recontruction loss, respectively.
% To support the optimization of $\mathcal{L}_{\mathrm{gtc}}$ and $\mathcal{L}_{\mathrm{rec}}$, we introduce auxiliary decoders $\theta_2$ and classifiers $\xi$. Each local model is defined as $\theta_u = \{\theta_{E,u}, \theta_{K,u}, \theta_{V,u}, \theta_{D,u}, \theta_{C,u}\}$, where $\theta_{E,u}$ is the shallow feature extractor, $\theta_{K,u}$ and $\theta_{V,u}$ are the causal and non-causal encoders, $\theta_{D,u}$ is the decoder, and $\theta_{C,u}$ is the classifier at client $u$.

% On the server side, the global model $\theta$ is obtained by averaging the local model parameters $\theta_u$ from all participating clients $u \in \mathcal{U}$. The resulting global model is then broadcast to all clients for the next round of local training.
Overall, the central idea is to impose representation disentanglement within the model, where two distinct latent components are learned: causal representations, which are invariant and informative for classification, and non-causal representations, which capture variant information sensitive to data perturbations. Together, these two representations contribute to accurate data reconstruction.
Based on this formulation, the proposed L2U stage is expressed as a bi-level optimization problem:
\begin{align}
&\min_{\theta} \left[\gL_{\mathrm{L2U}}(\theta) = \frac{1}{U}\sum_{u=1}^{U}\gL_{\mathrm{L2U}}(\theta_u)\right], \\
\text{s.t. } &\gL_{\mathrm{L2U}}(\theta_u) = \min_{\theta_u} \left[\alpha_{\mathrm{V}}\gL_{\mathrm{V}}
+ \alpha_{\mathrm{K}}\gL_{\mathrm{K}}
+ \alpha_{\mathrm{gtc}}\gL_{\mathrm{gtc}}
+ \gL_{\mathrm{rec}}\right], \notag
\end{align}
where $\gL_{\mathrm{V}}$, $\gL_{\mathrm{K}}$, $\gL_{\mathrm{gtc}}$, and $\gL_{\mathrm{rec}}$ denote the non-causal, causal, ground truth classification, and reconstruction losses, respectively.

To facilitate the optimization of $\gL{\mathrm{gtc}}$ and $\gL{\mathrm{rec}}$, auxiliary decoders ($\theta_2$) and classifiers ($\xi$) are introduced. Each local model is parameterized as $\theta_u = \{\theta_{E,u}, \theta_{K,u}, \theta_{V,u}, \theta_{D,u}, \theta_{C,u}\}$, where $\theta_{E,u}$ represents the shallow feature extractor; $\theta_{K,u}$ and $\theta_{V,u}$ are the causal and non-causal encoders, respectively; $\theta_{D,u}$ is the decoder; and $\theta_{C,u}$ is the classifier of client $u$.

The detailed two-stage optimization procedure for L2U is presented in Appendix~\ref{sec:two-stage-optimization}, and the overall structure of the local model is depicted in Fig.~\ref{fig:FOUL1}.

\paragraph{Training Causal Featurizer.} The goal of the causal featurizer is to extract class-discriminative features that are invariant to domain-specific, non-causal variations. To enforce this invariance, we train the model to generate highly compact representations for all samples belonging to the same class. By minimizing this intra-class distance, the model is encouraged to ignore superficial differences and focus on the essential, shared characteristics. The prototypical network framework \cite{NIPS2017_cb8da676} is exceptionally well-suited for this objective, as it explicitly learns a central prototype for each class. To achieve this, we train the causal performer using the prototypical network framework in the following manner:
\begin{align}
    \mathcal{L}_{\textrm{K}} = - \sum^{C}_{c=1} \log \frac{\exp(s(f_\phi (x),p_c))}{\sum^{C}_{i=1} \mathbbm{1}_{[i\neq c]} \exp(s(f_\phi (x),p_{i}))},
\end{align}
% where $-d(f_\phi (x),p_c)$ and $-d(f_\phi (x),p_{i})$ is the 
where $s(f_\phi (x),p_c)$ represents the cosine similarity between the encoded representations $f_\phi (x)$ and the approximated class prototype $p_c$. To approximate the prototype within the batch, we compute as \cite{NIPS2017_cb8da676} as $p_c = \frac{1}{|\gS_c|}\sum_{s_i,y_i\in \gS_c} f_{\phi}(x_i)$. $B_c\in B$ is the batch of data belonging to label $c$.

% \node[draw, circle, fill=black, text=white, inner sep=0.1pt](char){b}

\paragraph{Training Non-causal Featurizer.} To extract non-causal representations effectively, we maximize the empirical distance among representations within a specific label. However, two problems may arise during the maximization of the MSE between samples: \noindent\tikz[baseline=(char.base)] \node[draw, circle, fill=black, text=white, inner sep=0.1pt](char){a}; When the MSE value is at its initial value (where the variance is relatively low, i.e., nearly $0$), the optimization process prioritizes tasks with high initial losses. This prioritization can overshadow the learning of variant representations in the early stages. \noindent\tikz[baseline=(char.base)] \node[draw, circle, fill=black, text=white, inner sep=0.1pt](char){b}; When the MSE value is already optimized, performing the maximization process further can increase the variance, potentially diminishing the learning of other tasks.
To mitigate these problems, we adopt the hinge loss framework introduced in \cite{2022-IL-VICReg} in designing the variant loss $\mathcal{L}_{\textrm{V}}$:
\begin{align}
    \mathcal{L}_{\textrm{V}} =
    % \frac{1}{C}\sum^C_{c=1}\mathcal{L}^c_{\textrm{V}} = 
    \frac{1}{C}\sum^C_{c=1} \max\Big(0, 1 - \sqrt{\var(\gZ^{ c}_{{V}})+\epsilon}\Big),
\end{align}
where $\var(\gZ^{c}_{{V}})$ represents the variance among samples in the same class $c$, and
\begin{align}
    \var(\gZ^{ c}_{{V}}) = \mathbb{E}_{i \in  B_c} \left[\Big\Vert z^{(i)}_{{V}} - \mathbb{E}_{j \in  B_c}\Big(z^{(j)}_{{V}}\Big)\Big\Vert^2\right].
\end{align}
Incorporating the hinge loss brings the non-causal learning function into better alignment with other loss functions used in FOUL. This alignment involves initiating the optimization from higher values to mitigate trivial losses during the initial phase, while setting a lower boundary at zero to prevent bias towards $\mathcal{L}_{\textrm{V}}$ as the learning progresses. 

\paragraph{Meaningful Representations.}
From Definition~\ref{def:csm}, the disentangled factors must be sufficient for the reconstruction task. Hence, our goal for training FOUL is twofold: 
\begin{itemize}
    \item The causal representations capture crucial information that accurately represents their respective label.
    \item The two extracted factors can be combined and collectively used to reconstruct the original data with the highest possible accuracy.
\end{itemize}
To achieve the first goal, it should be possible to predict a specific label from the causal factors $\gZ_K$. We introduce the label classification function $\mathcal{L}_{\textrm{gtc}}$ as follows:
\begin{align}
    \mathcal{L}_{\textrm{gtc}} = -\mathbb{E}_{i\in B} \left[P(y^{(i)}) \log P\Big(f_{\xi}(z^{(i)}_{K})\Big)\right],
\end{align}
where $f_{\xi}:\mathbb{R}^{d_{z_K}}\rightarrow\mathbb{R}$ is the classifier network, parameterized by $\xi$, handling the classification task for the causal factors $z^{(i)}_{K}$. $y^{(i)}$ is the label of data instance $i$. $P(\cdot)$ denotes the probability distributions.

\begin{comment}
However, FOUL cannot extract causal factors if it is optimized by the classification loss solely. The reason is that crucial information may be memorized in the classifier $\xi$, which causes overfitting when operating in new users \cite{2020-DG-EntropyReg}. This problem can be alleviated by exploiting the adversarial learning \cite{2018-DG-CIAN}. Specifically, we introduce an adversarial discriminator $\Psi$, and train the classifier and discriminator as a min-max fairness game as follows: 
\begin{align}
    \min_{\theta_K, \theta_E}&\max_\Psi \mathcal{L}_\textrm{adv} 
    = \underset{i\in  B}{\mathbb{E}} \left[P(y^{(i)}) \log\frac{1}{P\Big[f_{\Psi}(z^{(i)}_{K})\Big]}\right], \label{eq:adversarial}\\
    &~\textrm{s.t.} ~~~z^{(i)}_{K} =f_{\theta_{K}}(f_{\theta_{E}}(x^{(i)})), \notag 
\end{align}
where $f_{\Psi}:\mathbb{R}^{d_{z_K}}\rightarrow\mathbb{R}$, $\theta_E$ denotes the shallow-layer encoder, which is designed to reduce redundant information. The purpose of~\eqref{eq:adversarial} is to stimulate knowledge generated by FOUL to deceive the adversarial discriminator $\Psi$, whereas the discriminator $\Psi$ endeavors to identify any false information in the causal factors $z_K$.
\end{comment}

To achieve the second goal, we employ the reconstruction loss $\mathcal{L}_{\textrm{rec}}$ according to \cite{2013-DL-VAE}. Specifically, we minimize the MSE between the original and its reconstructed data, while maximizing the KL divergence, and we enhance generalization by compelling the featurizer to output a distribution over the latent space instead of a single point. For instance,
\begin{align}
    \mathcal{L}_{\textrm{rec}}
	&= \gL_{\textrm{MSE}} - \gL_{\textrm{KL}}.
\end{align}

\subsection{Client-wise Unlearning} Besides the unlearning of only the non-causal featurizers, our objective is to propose an efficient strategy to aggregate knowledge from all participated users without the explosion of communication overheads. Therefore, the unlearned global model of the FL system may reduce the generalization gap with current centralized machine unlearning methods. The intuition behind our work is inherited from the invariant gradient direction of \cite{2022-DG-Fish, 2024-FDG-FedOMG, 2024-DG-POGM}.
\begin{assumption}[Invariant Gradient Direction \cite{2024-FDG-FedOMG}]
    Consider a model $\theta$ with task of finding domain-invariant features, the features generated by the model $\theta$ is domain-invariant if the two gradients $\nabla\mathcal{L}(\theta;\gD_u), \textrm{ and} \ \nabla\mathcal{L}(\theta;\gD_v)$, on different domains $\gD_u$ and $\gD_v$ 
     point to a similar direction, i.e., $
     \left< \nabla\mathcal{L}(\theta;\gD_u), \nabla\mathcal{L}(\theta;\gD_v) \right> >0$. 
\label{ass:igd}
\end{assumption}
Assumption~\ref{ass:igd} recommends a strategy for efficiently aggregating knowledge of distributed users for the FL system to achieve generalization. From Assumption~\ref{ass:igd}, we can have the inverted assumption as follows:
\begin{assumption}[Unlearning via Gradient Matching]
    Considering a model $\theta$ with task of unlearning features, the features generated by the model $\theta$ suffers from negative transfer if the two gradients $\nabla\mathcal{L}(\theta;\gD_u), \textrm{and} \ \nabla\mathcal{L}(\theta;\gD_v)$ does not point to a similar direction, i.e., $ \left< \nabla\mathcal{L}(\theta;\gD_u), \nabla\mathcal{L}(\theta;\gD_v)\right><0$.
\label{ass:ugm}
\end{assumption}
Combining the two Assumptions~\ref{ass:igd} and \ref{ass:ugm}, we come up with an idea to design an on-server unlearning approach based on gradient matching (see Fig.~\ref{fig:FOUL2}). Specifically, on the server, our target is to find an unlearning global gradient that satisfies Assumption~\ref{ass:igd} on the retain user set $\gU_\gR$ while satisfying Assumption~\ref{ass:ugm} on the forget user set $\gU_\gF$.
\begin{figure}[!h]
\centering
\includegraphics[width = \linewidth]{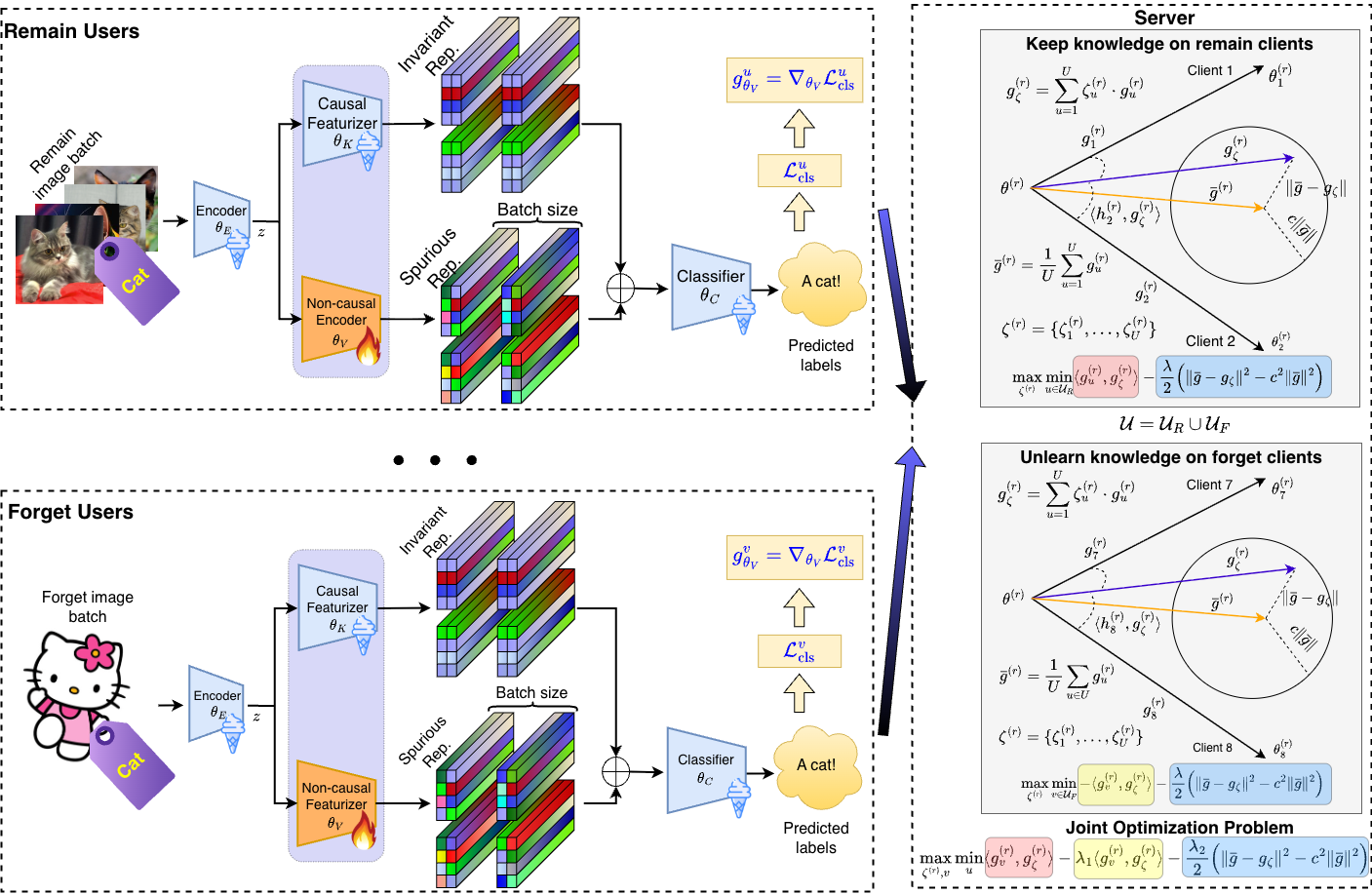}
\caption{Unlearning stage with data-free on-server unlearning. Only the non-causal featurizer participates in this phase. The gradients of the non-causal featurizer from both retain and forget users are sent to the server for on-server gradient matching.}
\label{fig:FOUL2}
\end{figure}
\subsubsection{On-server gradient-based unlearning}
Our target is to find an aggregated gradient so that it can make invariant directions towards gradients from the remaining clients $\gU_R$, while making negative transfer towards gradients in forget clients $\gU_F$. To reduce the computation cost of the optimization, we find the weights for the gradients aggregation so that it can form the optimal directions $g_{\textrm{FOUL}}$. For instance.
\begin{subequations}
\label{eq:igd-based-FL}
\begin{alignat} {3}
    &   &	&  ~~~~\theta^{(r)}_{v} = 
    \theta^{(r-1)}_{v} 
    - \eta g_{\textrm{FOUL}}, 
    \label{subeqn:igd-meta-obj}\\
    & \text{s.t.}
    &	& ~~~~ g_{\textrm{FOUL}} = \argmax_{g} 
                 \Big[\sum_{u\in \gU_\gR} \Big\langle g,\nabla(\phi^{(r)}_{u};\gD_u)\Big\rangle \notag \\
    & & & ~~~~~~~~~~~~~~~~~~~~~~~~~
               - \sum_{v\in \gU_\gF} \Big\langle g,\nabla(\phi^{(r)}_{v};\gD_v)\Big\rangle\Big], \label{subeqn:ign-angle}\\
    & & & ~~~~ \phi^{(r)}_{u} = \theta^{(r-1)}_{g} - \eta_l \nabla(\theta^{(r-1)}_{g};\gD_u), \forall u\in\gU.
    \label{subeqn:igd-local-obj} 
\end{alignat}
\end{subequations}
By proposing~\eqref{eq:igd-based-FL}, we establish the meta-objective function in~\eqref{subeqn:igd-meta-obj}, which relies solely on client gradients as training signals. 
The optimization in~\eqref{subeqn:ign-angle} aims to maximize the cosine similarity between the aggregated gradient and the gradients of the retain clients, while minimizing the cosine similarity with the gradients of the forget clients. 
This mechanism enforces unlearning on the forget clients while preserving performance on the retain clients.
However, one limitation of~\eqref{eq:igd-based-FL} is that the optimization over the model parameters $g$ may exhibit poor generalization, since it operates without access to explicit training data at each optimization step. 
This issue becomes particularly severe for models with large parameter dimensionality. 
To mitigate this problem, we introduce a set of learnable weighting coefficients 
$\Gamma = \{\gamma^{(r)}_{u}\mid u\in\gU,~ \sum_{u\in \gU}\gamma^{(r)}_u = 1\}$, 
and reformulate the global gradient as $g = \sum_{u\in\gU}\gamma^{(r)}_{u}\,\nabla\mathcal{L}(\phi^{(\phi^{(r)}_{u};\gD_u)})$.
This parameterization substantially reduces the optimization space: instead of learning the full parameter vector of dimension $d$, we now only optimize over $U$ coefficients, where $U \ll d$. Based on this reformulation, we derive the following theorem.
\begin{theorem}[FOUL solution] \label{theorem:FOUL}
    Given $\Gamma = \{\gamma^{(r)}_{u}\vert u\in\gU, \sum_{u\in \gU}\gamma^{(r)}_u = 1\}$ is the set of learnable coefficients at each round $r$, where $\gU = \{\gU_\gR, \gU_\gF\}$. Optimal unlearning gradient $\nabla^{(r)}_\textrm{FOUL}$ is characterized as follows:
    \begin{align}
        &\nabla^{(r)}_\textrm{FOUL} = \nabla^{(r)}_{\textrm{FL}} + \frac{\kappa\Vert \nabla^{(r)}_{\textrm{FL}}\Vert}{\Vert \nabla^{(r)}_{\Gamma_\gR} - \nabla^{(r)}_{\Gamma_\gF}\Vert}(\nabla^{(r)}_{\Gamma_\gR} - \nabla^{(r)}_{\Gamma_\gF}), \\ \textrm{s.t.}\quad 
        &\Gamma_{\gR}^*, \Gamma_{\gF}^* = \arg\min_{\Gamma_\gR, \Gamma_\gF} (\nabla^{(r)}_{\Gamma_\gR} - \nabla^{(r)}_{\Gamma_\gF})\cdot \nabla^{(r)}_{\textrm{FL}} \notag \\
        &~~~~~~~~~~~~~~~~~~~~~~~~
        +\sqrt{\kappa}\Vert \nabla^{(r)}_{\textrm{FL}}\Vert\Vert \nabla^{(r)}_{\Gamma_\gR} - \nabla^{(r)}_{\Gamma_\gF}\Vert. \notag
    \end{align}
    % We denote $\Gamma^*$ as the optimal parameter set at round $r$.
\end{theorem}
Theorem~\ref{theorem:FOUL} gives us an on-server optimization approach via the gradient matching without the need for local data accessibility. The detailed algorithm of the unlearning stage of FOUL is detailed in Appendix~\ref{app:FOUL-pseudo}.

% \subsubsection{On-server gradient-based unlearning}

% \subsubsection{Causality-based Contribution Dampening}

%% file: sec/6_experiments.tex
\section{Experimental Setup}
\textbf{Unlearning Scenarios.} To effectively simulate the client data unlearning process, we utilize domain generalization datasets \cite{gulrajanisearch}. In particular, this dataset is partitioned into multiple distinct domains, which we use to represent different clients. We assign the retain and forget clients to separate domains to ensure a clear distinction between them. Consequently, when performing client unlearning, we can quantitatively assess the effectiveness of the unlearning process by evaluating model performance on both the retained and forgotten client domains, as well as on the test domain.

\noindent\textbf{Hyperparameters \& Datasets \& Model.} 
In our experiments, we evaluate the proposed method on four datasets: PACS \cite{li2017deeper}, VLCS \cite{2020-DG-DomainBed}, OfficeHome \cite{venkateswara2017deep}, and TerraIncognitta \cite{beery2018recognition}. Each dataset contains four distinct domains (i.e., Photo, Art, Sketch, and Cartoon for PACS, and Art, Clipart, Product, and Real-World for OfficeHome). We adopt ResNet-18 as the backbone network for PACS, OfficeHome, VLCS, and ResNet-50 for TerraIncognita. All experiments are conducted on a single NVIDIA A100 GPU.
To simulate the FL environment, we consider $U=20$ clients under an IID setting. The clients are divided across the $4$ domains, with $5$ clients assigned to each domain. For the FUL experiments, we designate clients from one domain as the forget set, while clients from the remaining domains constitute the retain set. The unlearning procedure is then applied to all clients within the forget set.
% Further details are in Appendix.

\noindent\textbf{Implementation Details.} For every dataset and model configuration, we initially train a global model using FedAvg until it reaches convergence, after which the proposed unlearning algorithm is applied. In each round, all clients participate according to the designated retain and forget sets. The local dataset of each client is partitioned into training and validation subsets in a 90:10 ratio, while an independent test set is employed to evaluate the model’s performance.

\noindent\textbf{Evaluation Metrics.} Following prior works~\cite{gu2024ferrari, khalil2025not}, we evaluate the effectiveness and efficiency of FOUL using the following metrics: retain accuracy (RA), forget accuracy (FA), test accuracy (TA), membership inference attack (MIA), communication cost, and computation cost. The retain and forget accuracies measure the model’s performance on the retain and forget dataset $\mathcal{D}_r, \mathcal{D}_f$, respectively. MIA quantifies the extent to which $\mathcal{D}_f$ remain recognizable to the model after unlearning. The retain accuracy reflects how well the model preserves knowledge of $\mathcal{D}_r$, the forget accuracy evaluates how effectively the model forgets $\mathcal{D}_f$. Communication and computation costs represent the total amount of data transmitted and the total FLOPs required every communication round during the unlearning process.
In addition, we introduce Time-to-Forget (T2F), which measures the average reduction in accuracy on the forget set per round, thereby indicating how quickly the model unlearns the designated clients. To ensure fair evaluation between the retain and forget sets, we assess performance using the global model obtained after each aggregation round and report the average results across all local datasets.

\noindent\textbf{Baselines.} We compare the proposed FOUL method with the baselines, including Retraining (in which the model is retrained from scratch using only the retain client set. The resulting aggregated models are then evaluated on both the retain and forget client sets), FATS \cite{2024-FUL-PFU}, FedDCP \cite{wang2022federated}, FedRecovery \cite{zhang2023fedrecovery}, FedOSD \cite{pan2025federated}, FFMU \cite{2023-FUL-NFT}, FUSED \cite{zhong2025unlearning}, MoDE \cite{2024-FUL-MD}, and NoT \cite{khalil2025not}. 
% Additional details are provided in Appendix.

\section{Experimental Evaluation}
\paragraph{Main Results.}

\begin{table*}[!h]\small
\centering
\caption{Comparison of methods on ResNet-18 with PACS and TerraIncognita datasets. The table reports Forget Accuracy (FA), Remaining Accuracy (RA), Testing Accuracy (TA), and Membership Inference Attack (MIA), with values in parentheses showing the difference from the Retrain baseline. For the FOUL method, (1) corresponds to the L2U phase, and (2) refers to the on-server gradient matching unlearning phase. Note that we \textbf{bold} results achieving the closest TA accuracy to Retrain.}
\label{tab:results_pacs_office}
\resizebox{\textwidth}{!}{
\begin{tabular}{@{}l@{\;}l@{}c@{\;\;}c@{\;\;}c@{\;\;}c@{\quad}c@{\;\;}c@{\;\;}c@{\;\;}c@{}}
\toprule
\multirow{2}{*}{Group} & \multirow{2}{*}{Methods} & \multicolumn{4}{c}{PACS} & \multicolumn{4}{c}{TerraIncognita} \\
\cmidrule(lr){3-6} \cmidrule(lr){7-10}
 & & FA $(\downarrow)$ & RA $(\uparrow)$ & TA $(\uparrow)$ & MIA $(\downarrow)$ & FA $(\downarrow)$ & RA $(\uparrow)$ & TA $(\uparrow)$ & MIA $(\downarrow)$ \\
\midrule
\multirow{3}{*}{Retrain} & Retrain & $70.51$ & $82.84$ & $77.45$ & $50.02$ & $30.64$ & $42.41$ & $38.94$ & $50.71$ \\
& FATS & $74.45$ \textcolor{blue}{($+3.94$)} & $80.91$ \textcolor{blue}{($-1.93$)} & $75.98$ \textcolor{blue}{($-1.47$)} & $55.72$ \textcolor{blue}{($+5.70$)} & $33.07$ \textcolor{blue}{($+2.43$)} & $40.96$ \textcolor{blue}{($-1.45$)} & $37.34$ \textcolor{blue}{($-1.60$)} & $60.81$ \textcolor{blue}{($+10.10$)} \\
& NoT & $73.24$ \textcolor{blue}{($+2.73$)} & $79.25$ \textcolor{blue}{($-3.59$)} & $75.28$ \textcolor{blue}{($-2.17$)} & $59.13$ \textcolor{blue}{($+9.11$)} & $32.26$ \textcolor{blue}{($+1.62$)} & $38.32$ \textcolor{blue}{($-4.09$)} & $36.08$ \textcolor{blue}{($-2.86$)} & $62.90$ \textcolor{blue}{($+12.19$)} \\
\cmidrule(lr){1-10}
\multirow{6}{*}{\shortstack{Appro.\\Unlearn.}} & FedCDP & $77.37$ \textcolor{blue}{($+6.86$)} & $78.13$ \textcolor{blue}{($-4.71$)} & $76.41$ \textcolor{blue}{($-1.04$)} & $72.26$ \textcolor{blue}{($+22.24$)} & $34.56$ \textcolor{blue}{($+3.92$)} & $38.79$ \textcolor{blue}{($-3.62$)} & $36.75$ \textcolor{blue}{($-2.19$)} & $81.91$ \textcolor{blue}{($+31.20$)} \\
& FedRecovery & $76.48$ \textcolor{blue}{($+5.97$)} & $76.97$ \textcolor{blue}{($-5.87$)} & $74.81$ \textcolor{blue}{($-2.64$)} & $75.24$ \textcolor{blue}{($+25.22$)} & $32.11$ \textcolor{blue}{($+1.47$)} & $37.49$ \textcolor{blue}{($-4.92$)} & $35.59$ \textcolor{blue}{($-3.35$)} & $84.03$ \textcolor{blue}{($+33.32$)} \\
& FedOSD & $72.89$ \textcolor{blue}{($+2.38$)} & $80.15$ \textcolor{blue}{($-2.69$)} & $75.49$ \textcolor{blue}{($-1.96$)} & $58.84$ \textcolor{blue}{($+8.82$)} & $32.63$ \textcolor{blue}{($+1.99$)} & $40.77$ \textcolor{blue}{($-1.64$)} & $36.76$ \textcolor{blue}{($-2.18$)} & $60.54$ \textcolor{blue}{($+9.83$)} \\
& FFMU & $73.31$ \textcolor{blue}{($+2.80$)} & $78.27$ \textcolor{blue}{($-4.57$)} & $74.14$ \textcolor{blue}{($-3.31$)} & $60.62$ \textcolor{blue}{($+10.60$)} & $33.76$ \textcolor{blue}{($+3.12$)} & $39.39$ \textcolor{blue}{($-3.02$)} & $36.04$ \textcolor{blue}{($-2.90$)} & $65.85$ \textcolor{blue}{($+15.14$)} \\
& FUSED & $75.94$ \textcolor{blue}{($+5.43$)} & $79.34$ \textcolor{blue}{($-3.50$)} & $76.86$ \textcolor{blue}{($-0.59$)} & $58.72$ \textcolor{blue}{($+8.70$)} & $33.64$ \textcolor{blue}{($+3.00$)} & $38.73$ \textcolor{blue}{($-3.68$)} & $36.74$ \textcolor{blue}{($-2.20$)} & $62.03$ \textcolor{blue}{($+11.32$)} \\
& MoDE & $72.53$ \textcolor{blue}{($+2.02$)} & $79.01$ \textcolor{blue}{($-3.83$)} & $75.79$ \textcolor{blue}{($-1.66$)} & $59.79$ \textcolor{blue}{($+9.77$)} & $32.17$ \textcolor{blue}{($+1.53$)} & $38.04$ \textcolor{blue}{($-4.37$)} & $35.74$ \textcolor{blue}{($-3.20$)} & $63.47$ \textcolor{blue}{($+12.76$)} \\
\cmidrule(lr){1-10}
\multirow{2}{*}{FOUL} & (1) & $\mathbf{69.53}$ \textcolor{blue}{($-0.98$)} & $\mathbf{93.11}$ \textcolor{blue}{($+14.55$)} & $\mathbf{77.14}$ \textcolor{blue}{($-0.31$)} & $53.82$ \textcolor{blue}{($+3.80$)}
& $\mathbf{27.97}$ \textcolor{blue}{($-2.67$)} & $\mathbf{43.81}$ \textcolor{blue}{($+1.40$)} & $38.16$ \textcolor{blue}{($-0.78$)} & $\mathbf{56.40}$ \textcolor{blue}{($+5.69$)} \\
& (1) + (2) & $70.97$ \textcolor{blue}{($+0.46$)} & $92.33$ \textcolor{blue}{($+14.49$)} & $76.43$ \textcolor{blue}{($-1.02$)} & $\mathbf{51.93}$ \textcolor{blue}{($+1.91$)}
& $29.92$ \textcolor{blue}{($-0.72$)} & $42.13$ \textcolor{blue}{($-0.28$)} & $\mathbf{39.16}$ \textcolor{blue}{($+0.22$)} & $57.11$ \textcolor{blue}{($+6.40$)} \\
\bottomrule
\end{tabular}
}
\end{table*}
Table~\ref{tab:results_pacs_office} presents a comparison of FOUL with other baselines on the PACS and TerraIncognita datasets. The baselines are grouped into two categories: fast retraining and approximate unlearning. 
In general, fast retraining methods achieve satisfactory performance in terms of FA and RA. However, they exhibit low T2F, indicating a slower unlearning process. Consequently, these methods suffer from high computational and communication overheads, as well as longer convergence times.
Although approximate unlearning methods achieve higher T2F and faster convergence, they generally yield lower FA and RA compared to fast retraining approaches.
FOUL consistently outperforms all baselines in client-wise FUL across different architectures and datasets, achieving a smaller average performance gap while maintaining competitive communication and computation costs. The more results on VLCS and OfficeHome datasets are presented in Appendix~\ref{app:detailed-results}.

\begin{comment}
- \hl{Check if the diff between RA and TA okay}

- \hl{Check if you want TA to decrease more than RA}
\end{comment}

\paragraph{Time to Forget.}
Figures \ref{fig:RetrainNoReset-results}(a), \ref{fig:RetrainReset-results}(a), and \ref{fig:FOUL-results}(a) present a comparative analysis of unlearning performance across three approaches: FOUL, retraining with a pre-trained model, and retraining with model reset. The results demonstrate that the two naive retraining methods fail to achieve effective unlearning for the designated forget client set. This limitation arises because both retraining methods employ standard averaging aggregation at the server, which cannot explicitly remove knowledge associated with the forget clients. In contrast, FOUL achieves a substantial reduction in FA, reaching its optimal FA in fewer than 50 rounds, achieving T2F more than $0.32 / \text{round}$ (compare to results 75 rounds and T2F of $0.13$ of Retraining). Moreover, the RA of FOUL continues to improve during the unlearning phase, indicating that the knowledge of the remaining clients is preserved and not unintentionally unlearned.
% \begin{figure}
%     \centering
%     \includegraphics[width=0.99\linewidth]{image-lib/experimental-evaluations/FedAvgRC_R_Accuracy.pdf} \\ 
%     \includegraphics[width=0.99\linewidth]{image-lib/experimental-evaluations/FedAvgRC_R_Angle_Cosine.pdf}
%     \caption{Retraining with without resetting models with FedAvg on PACS dataset.}
%     \vspace{-2mm}
% \label{fig:RetrainNoReset-results}
% \end{figure}
\begin{figure}
    \centering
    \includegraphics[width=0.495\linewidth]{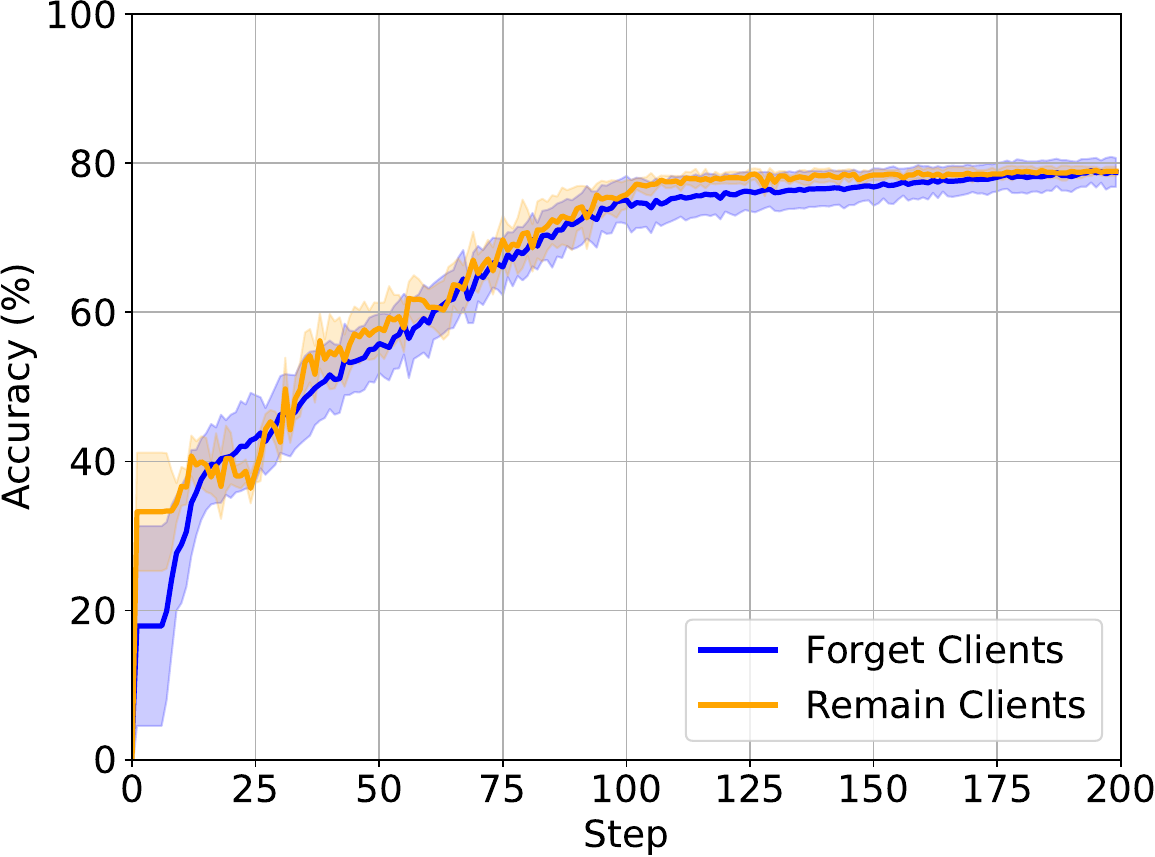} 
    \includegraphics[width=0.495\linewidth]{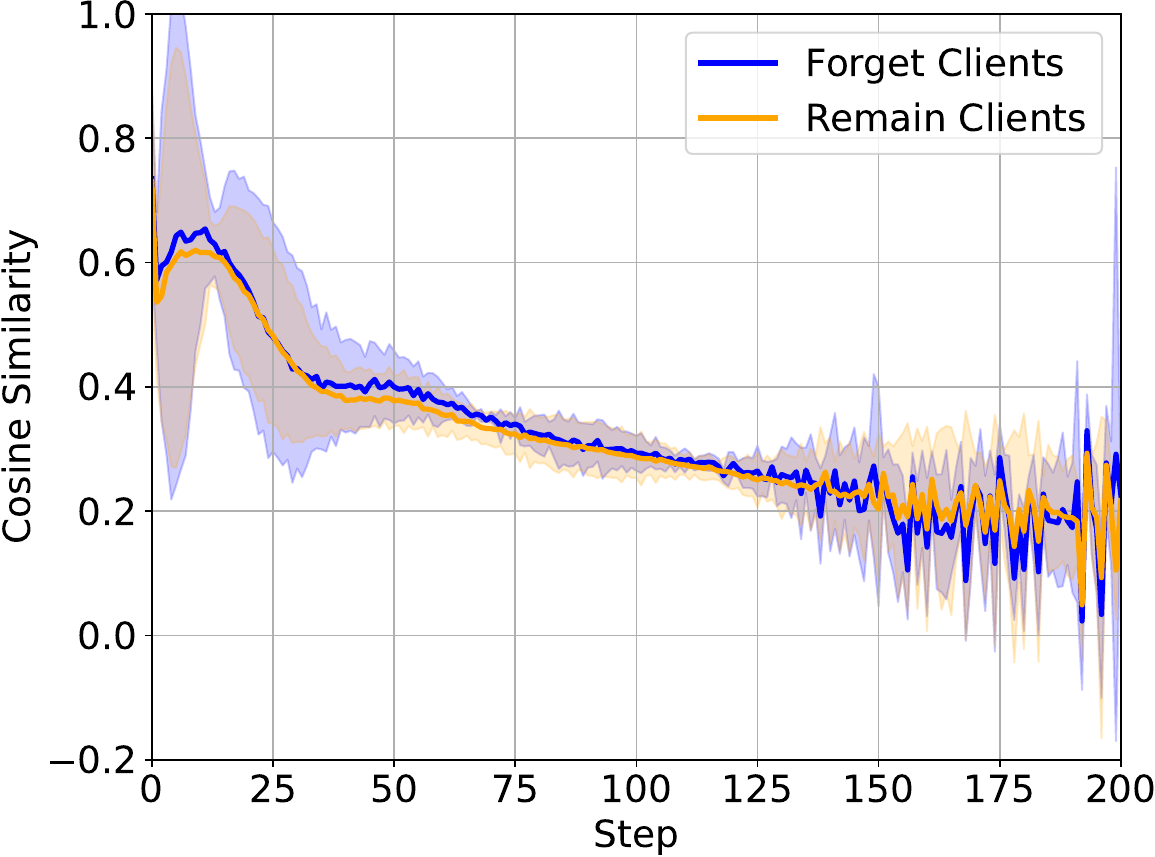}
    \caption{Retraining from scratch with FedAvg without resetting model parameters on PACS dataset.}
    \vspace{-4mm}
\label{fig:RetrainNoReset-results}
\end{figure}
\begin{figure}
    \centering
    \includegraphics[width=0.495\linewidth]{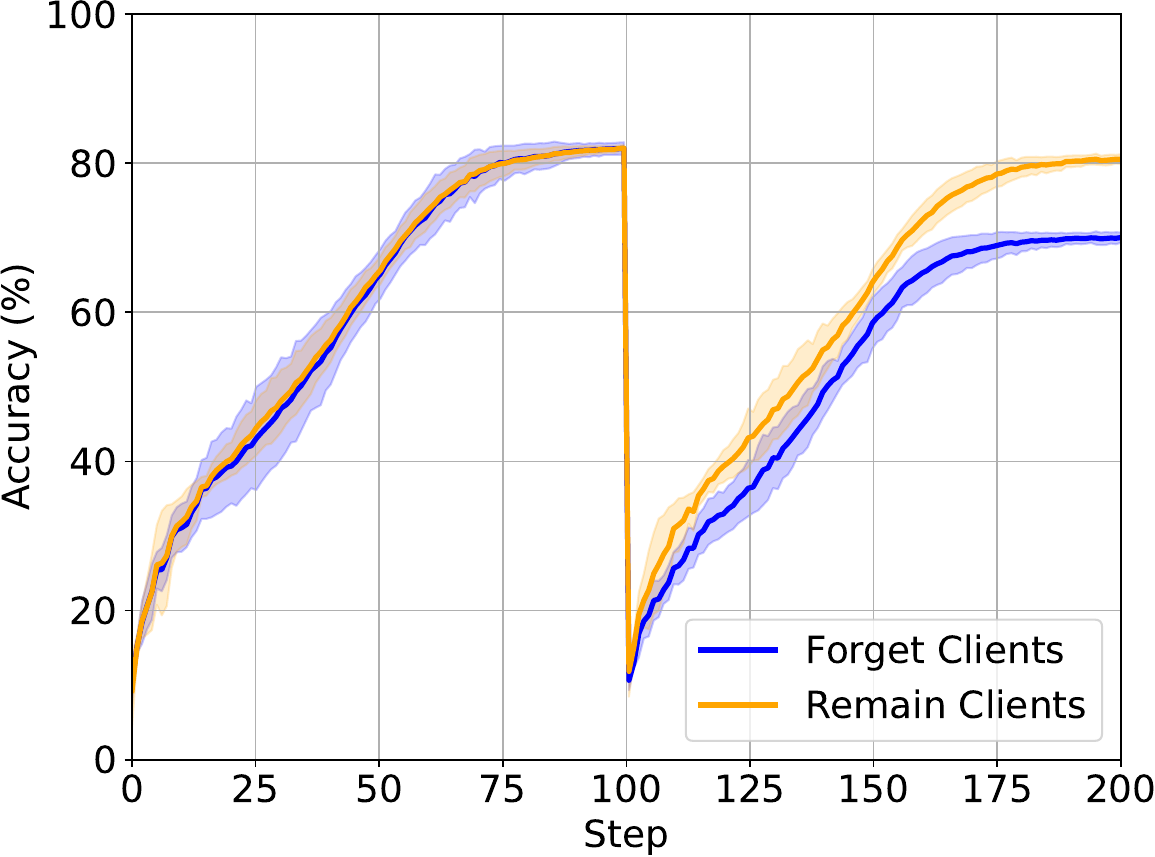} 
    \includegraphics[width=0.495\linewidth]{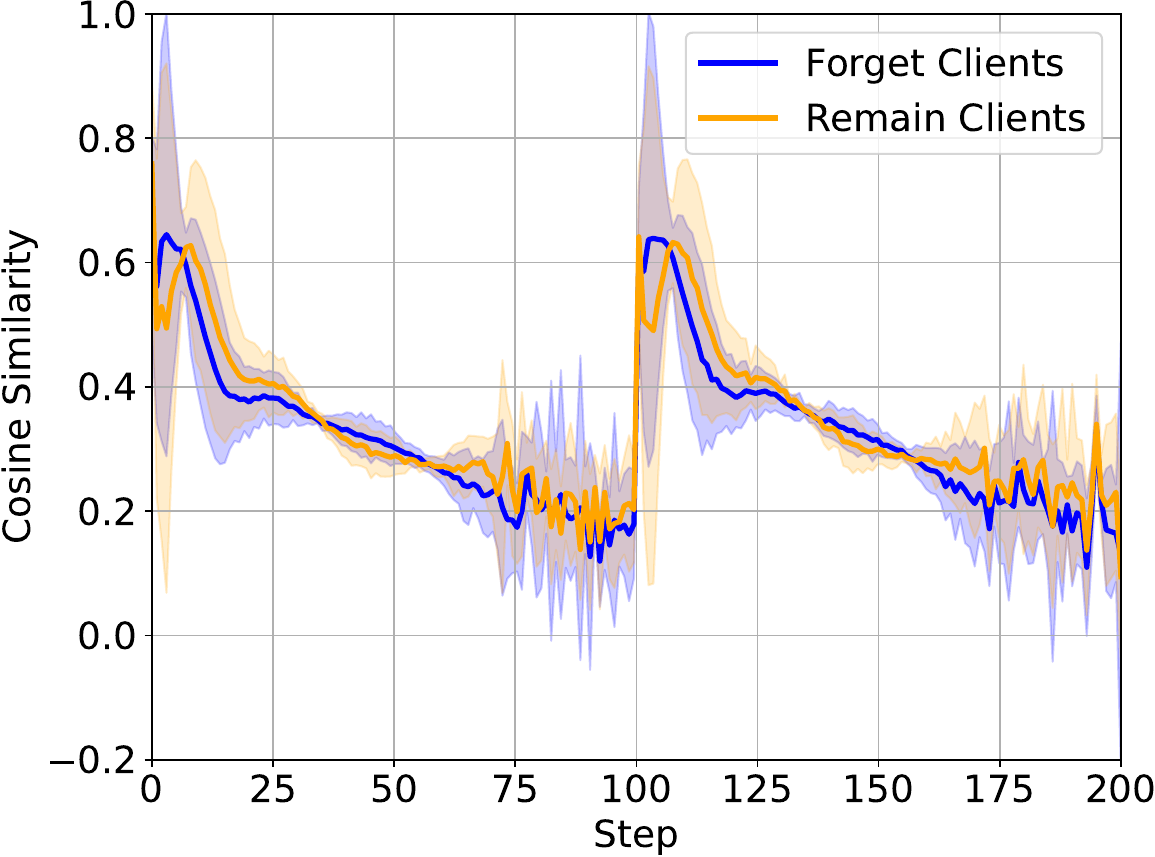}
    \caption{Retraining from scratch with FedAvg with resetting model parameters on PACS dataset.}
    \vspace{-4mm}
\label{fig:RetrainReset-results}
\end{figure}
\begin{figure}[t]
    \centering
    \includegraphics[width=0.495\linewidth]{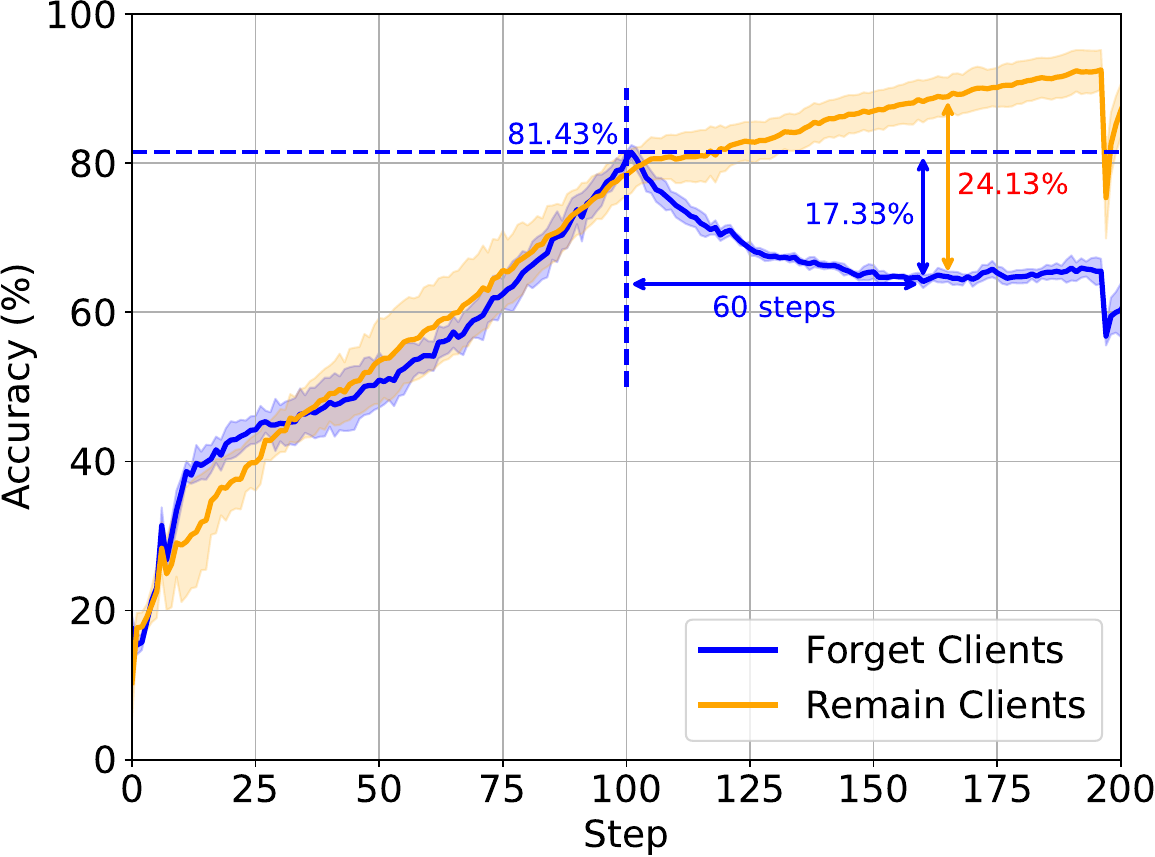}
    \includegraphics[width=0.495\linewidth]{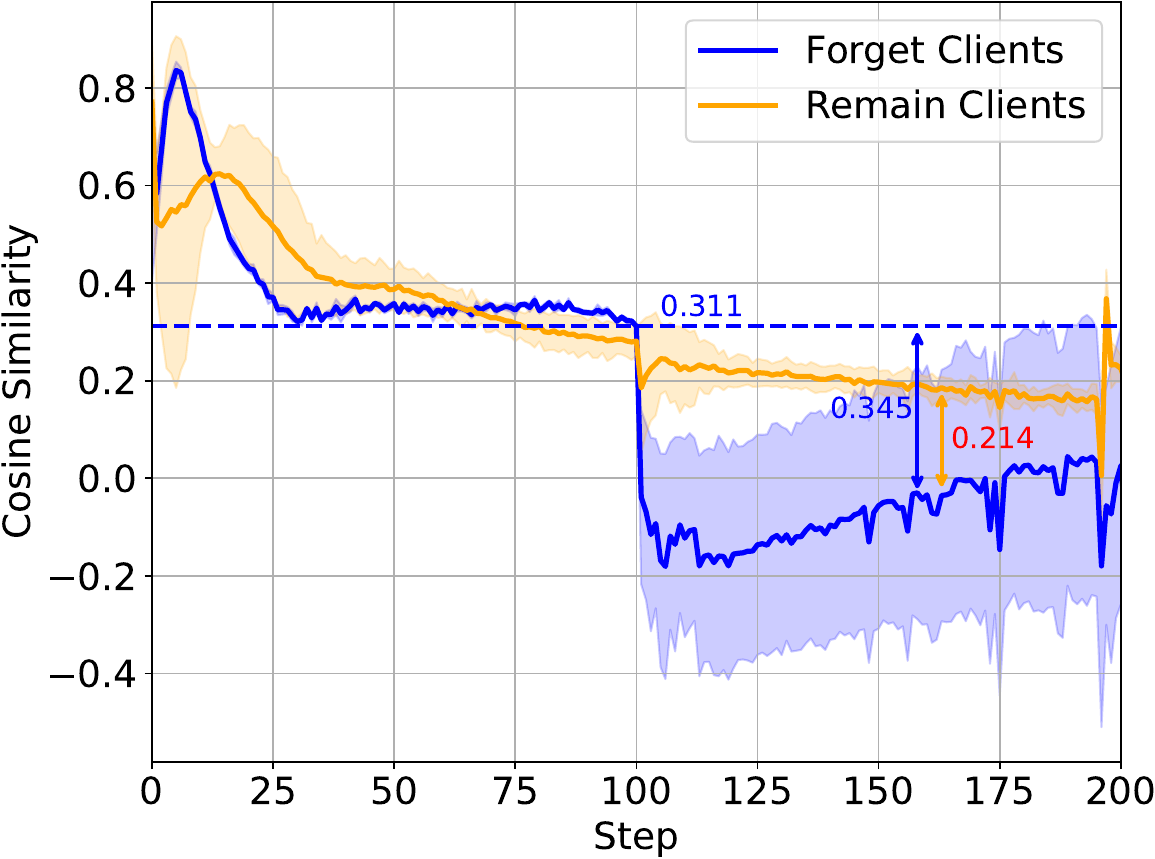}
    \caption{Evaluations of FOUL on PACS dataset.}
    \vspace{-4mm}
    \label{fig:FOUL-results}
\end{figure}

\begin{figure}[t]
    \centering
    \includegraphics[width=0.725\linewidth]{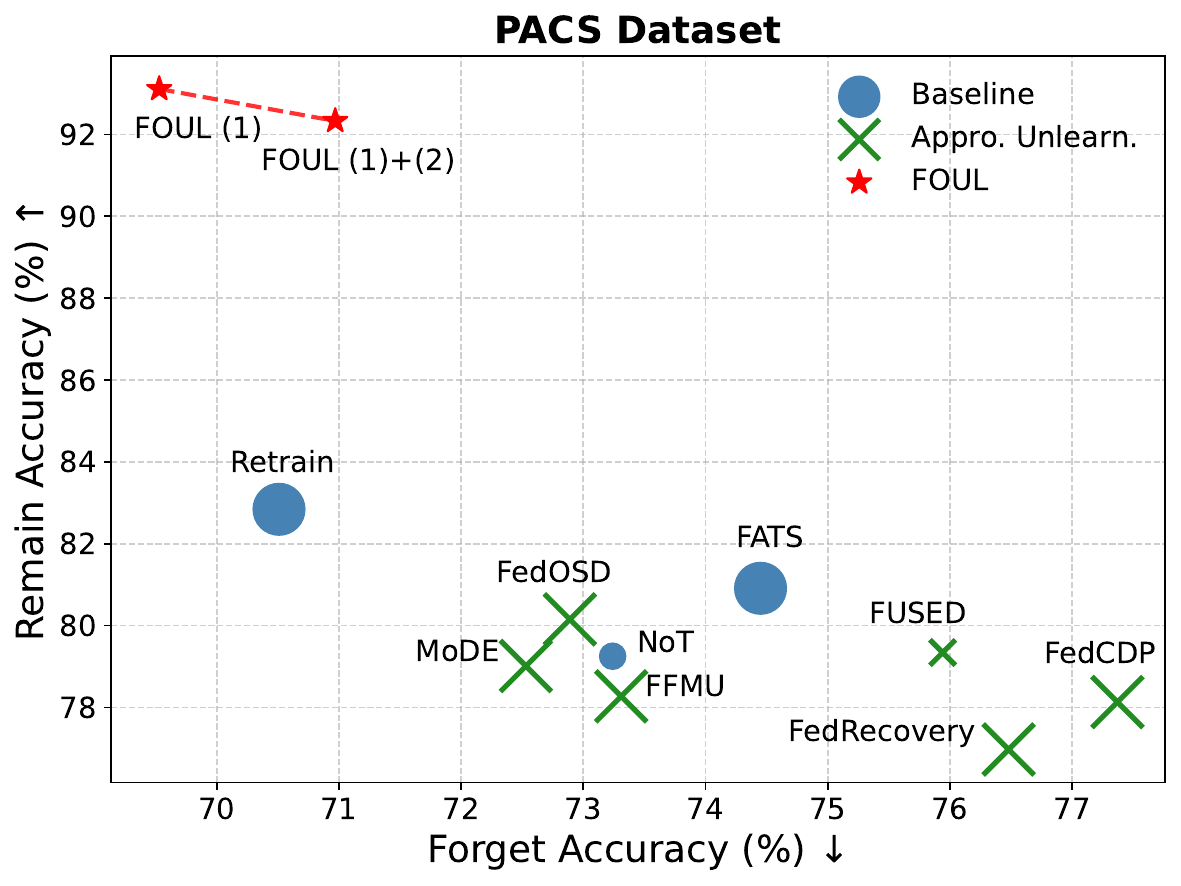}
    \includegraphics[width=0.725\linewidth]{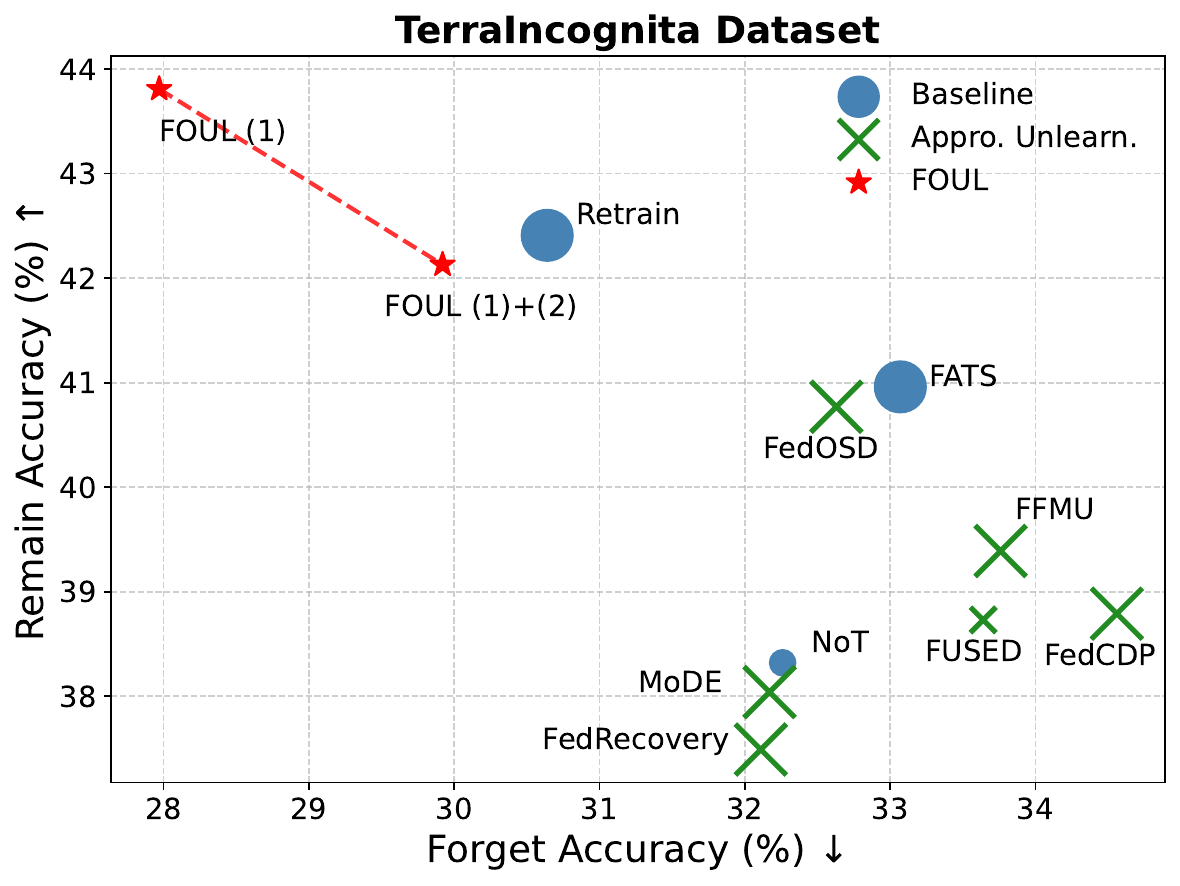}
    \caption{Comparison of communication cost in trade-off with Forget Accuracy (FA) and Remain Accuracy (RA) across different unlearning methods on PACS and TerraIncognita datasets.}
    \vspace{-2mm}
    \label{fig:FOUL-communication}
\end{figure}
% \begin{figure}
%     \centering
%     \includegraphics[width=0.99\linewidth]{image-lib/experimental-evaluations/FedOMGR_R_Accuracy.pdf} \\ 
%     \includegraphics[width=0.99\linewidth]{image-lib/experimental-evaluations/FedOMGR_R_Angle_Cosine.pdf}
%     \caption{Evaluations of FedOMG on PACS dataset.}
%     \vspace{-2mm}
% \label{fig:FedOMG-results}
% \end{figure}

\paragraph{Gradient angle and learning properties.}
Figures \ref{fig:RetrainNoReset-results}(b), \ref{fig:RetrainReset-results}(b), and \ref{fig:FOUL-results}(b) illustrate the average gradient divergence between the retain and forget client sets. This measure is obtained by computing the mean cosine similarity between the global model gradients and those of individual clients in both sets. As shown in Figures \ref{fig:RetrainNoReset-results}(b) and \ref{fig:RetrainReset-results}(b), there is no noticeable difference in gradient divergence between the retain and forget clients, which explains the limited unlearning performance observed in Figures~\ref{fig:RetrainNoReset-results}(a) and \ref{fig:RetrainReset-results}(a). In contrast, FOUL exhibits a pronounced reduction in FA that closely corresponds to a decrease in gradient similarity for the forget clients relative to the retain clients. Moreover, the negative cosine similarity values observed for the forget set align with the gradient conflict phenomenon, which has been shown to induce negative transfer effects in multi-task learning \cite{ban2024fair}, domain generalization \cite{2022-DG-Fish, 2022-DG-Fishr}.
% \begin{itemize}
%     \item The improvement in cosine similarity between local and global gradients exhibits a grouped pattern. Specifically, one set of clients demonstrates a noticeable increase in cosine similarity during stages 0 to 10, leading to faster convergence for these clients in the early phases. Conversely, another set of clients experiences rapid convergence during stages 10 to 25, characterized by a corresponding rise in cosine similarity. This behavior indicates a clear relationship between increasing cosine similarity and the convergence rate of the algorithm.
%     \item As the model progresses, the influence of cosine similarity diminishes, eventually stabilizing as the model approaches a steady learning state.
% \end{itemize}

% \subsection{Unlearning Robustness}
% Fig.~\ref{fig:unlearn-robustness} discuss about the unlearning robustness (by measuring the learned optimal $accuracy - forgetted accuracy$
% \begin{figure*}
%     \centering
%     \includegraphics[width=0.99\linewidth]{image-lib/template/Template-Robustness.png} 
%     \caption{Resilience of unlearning to downstream ...}
%     \vspace{-2mm}
%     \label{fig:unlearn-robustness}
% \end{figure*}

\paragraph{Communication \& Computation Efficiency.}
Figure~\ref{fig:FOUL-communication} illustrates the communication cost of the FOUL algorithm~\ref{app:FOUL-pseudo} compared with other baselines, as well as the trade-off between communication efficiency and RA, FA. As shown in the figure, integrating the L2U stage substantially reduces communication overhead. Moreover, FOUL achieves the highest RA and FA among all compared approaches. As summarized in Table~\ref{tab:cost_analysis}, FOUL also demonstrates superior communication and computation efficiency relative to baselines.

\begin{table}[!h]\small
\centering
\caption{Comparison of computational and communication characteristics of each method. 
Params denotes trainable parameters, Comm. refers to per-round communication overhead, and Comp. indicates local computation complexity.}
\label{tab:cost_analysis}
\resizebox{\columnwidth}{!}{
\begin{tabular}{lccc}
\toprule
\textbf{Method} & \textbf{Params (M)} & \textbf{Comm. (MB)} & \textbf{Comp. (FLOPs)} \\
                & \textbf{(M)} & \textbf{(MB)} & \textbf{(FLOPs)} \\
\midrule
Retrain        & 11.3 & 42.73 & $5.81e^{16}$ \\
FATS           & 11.3 & 42.73 & $5.81e^{16}$ \\
NoT            & 11.3 & 34.72 & $3.38e^{16}$ \\
FedCDP         & 11.3 & 42.73 & $5.82e^{16}$ \\
FedRecovery    & 11.3 & 42.73 & $5.96e^{16}$ \\
FedOSD        & 11.3 & 42.73 & $5.85e^{16}$ \\
FFMU           & 11.3 & 42.73 & $5.85e^{16}$ \\
FUSED          & 11.3 & \textbf{0.98} & $\underline{2.81e^{16}}$ \\
MoDE           & 11.3 & 42.73 & $3.72e^{16}$ \\
FOUL (1) + (2) & 11.3  & \underline{16.02} & $\mathbf{2.35e^{16}}$ \\
\bottomrule
\end{tabular}
}
\end{table}

\section{Ablation Test}
% \hl{We add / remove components. Normal network (ResNet), 
% UFOUL (with split and not disentanglement training), 
% UFOUL (with split and disentanglement training loss)}

\paragraph{Learning to Unlearn Phases} To evaluate the impact of the coefficients in the learning-to-unlearn stage of the FOUL algorithm, we conduct a series of experiments, with detailed results and discussions provided in Appendix~\ref{app:coefficient}. In addition, we assess the domain-invariant capability of the proposed model’s disentanglement mechanism, as presented in Appendix~\ref{app:invariant}. The results and discussion on the L2U stage with vary causal/non-causal representations dimensionality can be found in Appendix~\ref{app:knowledge-capability}. 

\paragraph{Unlearning Phase} 
To assess the effect of the coefficients used in the unlearning stage of the FOUL algorithm, we perform a series of experiments examining learning rate, sensitivity, and disentanglement efficiency. Results and discussions are provided in Appendix~\ref{app:knowledge-capability}.

% \begin{figure}[ht]
%     \centering
%     \includegraphics[width=0.495\linewidth]{image-lib/new_fig/learning_rate.pdf}\label{fig:lr}
%     \includegraphics[width=0.495\linewidth]{image-lib/new_fig/kappa.pdf}\label{fig:kappa}
%     \caption{Effects of different global learning rates and hypersphere radiuses.}
%     \vspace{-2mm}
% \end{figure}

%% file: sec/7_conclusion.tex
\section{Conclusion}
This paper addresses the unlearning problem in FL. To overcome key challenges in FUL (including forget-client data inaccessibility, indiscriminate knowledge removal, and high unlearning costs), we propose FOUL, an efficient two-stage framework for learning and unlearning. In the learning stage, the feature extractor is disentangled into two sub-networks: one encoding domain-invariant representations and another capturing domain-variant representations. By restricting the unlearning process to the domain-variant sub-network, FOUL preserves global, task-relevant knowledge across clients while effectively removing features associated with the forget set. In the unlearning stage, on-server gradient matching enables the utilization of both retain and forget client knowledge, significantly improving the time-to-unlearn efficiency. Experimental results demonstrate that FOUL outperforms Retraining in FUL tasks while maintaining communication and computation costs comparable to the standard FedAvg, primarily due to its ability toop leverage knowledge from the forget clients.

\section*{Acknowledgment} 
Minh-Duong Nguyen, Le-Tuan Nguyen, Dung D. Le are supported by the Center for AI Research, VinUniversity.

%% file: sec/b2_appendix.tex
\clearpage
\setcounter{page}{1}
\maketitlesupplementary

\section{Related Works}
\label{sec:related-works}
\subsection{Federated Unlearning}
Recently, there have be en some FUL approaches, including 1) fast retraining, and 2) approximate unlearning.
\paragraph{Fast retraining:} To design a fast training protocol for FUL, \cite{2024-FUL-DataRevocation, 2024-FUL-FR} propose user data revocation strategies to ensure the efficient unlearning process among users in FL system.
FedEraser \cite{2021-FUL-FedEraser} initiates the unlearning algorithm among the clients that request data removal. However, other clients' models still hold the contributions of the data to be removed, which will be later aggregated in a global model. To address this, their proposed solution uses the latest rounds of local model updates to approximate the current unlearned model update. Nonetheless, if the data samples to be removed have been used for training from the beginning, the model updates may still contain traces of these samples.
\cite{2022-FUL-RR} introduces a fast retraining approach by leveraging the diagonals of Fisher Information Matrix (FIM). This approach significantly improve the efficiency of the approximation of inverse Hessian. 
KNOT \cite{2023-FUL-AFU} proposes a clustered FL concept to improve the efficiency of retraining in FUL. Specifically, due to the clustered aggregation mechanism, the FUL concept using KNOT only has to retrain a partial set of clients, thus improving the retraining efficiency in FUL. NoT \cite{khalil2025not} introduces a weight negation operator that perturbs the model parameters away from their optimal values, thereby positioning the model more favorably for subsequent re-training.

\paragraph{Approximate Unlearning} 
FATS \cite{2024-FUL-PFU} evaluates the involvement of clients in the previous training of the federated learning (FL) system. Consequently, the contributions of clients are considered accordingly in the unlearning stage. MetaFUL \cite{2024-FUL-HumanCentric} proposes an efficient unlearning algorithm that leverages the impact of user data, which is measured during the learning stage. MoDe \cite{2024-FUL-MD} applies slow-update techniques in federated unlearning to improve the smoothness of the unlearning process, thereby enhancing overall performance. However, the naive unlearning approach (i.e., excluding clients with revocation requests from the aggregation stage) does not ensure a reduced generalization gap between conventional machine unlearning and FUL. 
FedRemover \cite{2024-FUL-RFU} introduces a two-step approach for unlearning malicious clients, which includes: 1) attack detection, and 2) gradient direction analysis-based malicious client identification. This method significantly improves energy and computational efficiency in identifying clients for unlearning. FFMU \cite{2023-FUL-NFT} employs quantization on local models, demonstrating that it can efficiently unlearn kno wledge from the set of forget devices. The authors in \cite{2022-FUL-CDP} propose a TF-IDF Guided Pruning approach for efficient FUL. Specifically, the server uses TF-IDF, a statistical measure of channel impact on class discrimination, to determine which network channels should be pruned. However, this pruning-based approach not only removes information about the unlearning classes but also erases information from some remaining data, leading to a degradation in data discrimination.

To address this issue, \cite{2024-FUL-USP} proposes an on-server unlearning process for FUL. However, this method requires on-server data, which negatively impacts the communication efficiency of the FL system. Conda \cite{2024-FUL-Conda} introduces an on-server unlearning method that identifies global model parameter sets most affected by forget clients. By dampening these parameters, Conda achieves effective unlearning while preserving the rest of the model parameters on the server side. 
FUSED \cite{zhong2025unlearning} identifies the critical layers that contribute most to learning and performs updates selectively on these layers through a sparse update strategy. This approach significantly reduces the computational cost required during the unlearning phase.

The authors in \cite{2024-FUL-Mixup} propose a mixup-based strategy to enrich the data used in unlearning. Ferrari \cite{gu2024ferrari} introduces a new metric for feature-level unlearning called feature sensitivity. FedME2 \cite{2023-FUL-FedME2} presents an FUL architecture with two key components: 1) memory evaluation (MEval), and 2) erasure (MErase). MEval follows the rationale that the informational bottleneck significantly reduces redundant information from the data. By evaluating the final layers, it determines whether data is efficiently remembered within the model architecture. The authors design an evaluation loss function for the federated unlearning problem, thereby efficiently improving the performance of federated unlearning. FedOSD \cite{pan2025federated} enables client-side unlearning by incorporating an unlearning cross-entropy loss to remove undesired knowledge locally, and adopts gradient surgery to ensure gradient alignment among heterogeneous clients.

\subsection{Prototype Learning}
Prototype learning \cite{NIPS2017_cb8da676} has emerged as a powerful paradigm for classification tasks, where models learn representative embeddings, or prototypes, for each class. A prototype typically refers to the mean feature vector of instances belonging to the same class \cite{2018-DG-CIAN}. Beyond serving as a classification tool, prototype learning can also be viewed as an implicit form of causal factor discovery, as each prototype may correspond to a distinct causal factor in the underlying CSM \cite{2022-DG-CIRL, huang2023idag}.

In supervised classification, a query sample is classified by measuring its distance to the prototypes of each class, which has been shown to yield more robust and stable performance \cite{zhang2023prototypical, saadabadi2024hyperspherical}, particularly in few-shot \cite{zhou2023revisiting, liu2022dynamic, cheng2022holistic, he2024apseg} and zero-shot learning \cite{hou2024visual, feinglass2024eyes} settings. Moreover, prototype-based approaches have gained increasing attention in semantic segmentation \cite{zhou2022rethinking, zhou2024prototype}, multimodal learning \cite{fan2023pmr} and unsupervised representation learning \cite{tang2024hunting}.

In the context of FL, prototypes can serve as a compact and privacy-preserving form of abstract knowledge that enables efficient information sharing among clients. Recent studies have integrated prototype learning into FL frameworks to address data heterogeneity across clients \cite{tan2022federated, zhang2024fedgmkd}, mitigate domain shift \cite{huang2023rethinking, wang2024taming}, and enhance personalization in local models \cite{tan2022federated, zhang2023gpfl, jia2024protofedla}.

\section{Two-stage Optimization for learning stage of FOUL} \label{sec:two-stage-optimization}
The joint loss function is complex which makes the optimization problem NP-hard. Acknowledging this shortcoming, we decompose the FOUL training process into two steps, denoted as the \emph{L2U Generator} and the \emph{Meaningful Invariant Discriminator} (MID). The \emph{L2U Generator} is responsible for training the featurizer to achieve good data reconstruction while maintaining the invariant and variant characteristics of $z_K$ and $z_V$, respectively. The \emph{MID} focuses on distilling essential knowledge into the invariant representation $z_K$. By splitting the complicated tasks into two simpler tasks with fewer loss components, the individual tasks become significantly more straightforward, thereby enhancing the overall training process. To train these two disentangled tasks jointly and efficiently, we use a cost-effective meta-learning technique called Reptile \cite{2018-MeL-Reptile}. Reptile is proven to be computationally efficient while maintaining performance comparable to conventional meta-learning approaches \cite{2017-Mel-MAML}.
\begin{algorithm}
\caption{Pseudo algorithm for L2U stage of FOUL.}
\label{alg:L2U}
\textbf{Input:} Initial model parameter $\theta_0$, learning coefficients $\alpha_\textrm{rec},\alpha_\textrm{v},\alpha_\textrm{iv},\alpha_\textrm{gtc},\alpha_\textrm{adv} \in [0,1)$ and learning rate $\eta_\textrm{L2U},\eta_\textrm{MID},\eta_\textrm{adv} \in\mathbb{R}^{+}$. Local iterations $E = I_\textrm{L2U} + I_\textrm{MID}$\\
\For{training round $r = 0,1,\dots,I $}{
    /*\textit{ Local Training }*/ \\
    \For{client $u\in \gU$}{
        /*\textit{ L2U Generator }*/ \\
        \For{iteration $i = 0,1,\dots,I_\textrm{L2U}-1 $}{
            Compute loss for the representation disentanglement:
            \begin{align}
                \mathcal{L}_{\textrm{L2U}} = \alpha_\textrm{rec}\mathcal{L}_{\textrm{rec}} + \alpha_\textrm{iv}\mathcal{L}_{\textrm{iv}} + 
                \alpha_\textrm{v}\mathcal{L}_{\textrm{v}}. \notag
            \end{align} \\ 
            Apply gradient descent according to~\eqref{eq:L2U-update}.
            % \begin{align}
            %     &\theta^{r,i+1}_{E} = \theta^{r,i}_{E} - \eta_\textrm{FOUL}\nabla_{\theta^{r,i}_{E}} \mathcal{L}_{\textrm{FOUL}} \notag \\
            %     &\theta^{r,i+1}_{V} = \theta^{r,i}_{V} - \eta_\textrm{FOUL}\nabla_{\theta^{r,i}_{V}} \mathcal{L}_{\textrm{FOUL}} \notag \\
            %     &\theta^{r,i+1}_{K} = \theta^{r,i}_{K} - \eta_\textrm{FOUL}\nabla_{\theta^{r,i}_{K}} \mathcal{L}_{\textrm{FOUL}} \notag \\
            %     &\theta^{r,i+1}_{2} = \theta^{r,i}_{2} - \eta_\textrm{FOUL}\nabla_{\theta^{r,i}_{2}} \mathcal{L}_{\textrm{FOUL}}. \notag
            % \end{align}
        }   
        % \emph{Meaningful Invariant Discriminator} \\
        /*\textit{ Meaningful Invariant Discriminator }*/ \\
        \For{iteration $ i = I_\textrm{L2U},\dots,I_\textrm{L2U}+I_\textrm{MID} $}{
            Compute loss for meaningful invariant representation:
            \begin{align}
                \mathcal{L}_{\textrm{MID}} =  \alpha_\textrm{iv}\mathcal{L}_{\textrm{iv}} + 
                \alpha_\textrm{gtc}\mathcal{L}_{\textrm{gtc}}. \notag
            \end{align} \\ 
            Apply gradient descent according to~\eqref{eq:MID-update}.
            % \begin{align}
            %     &\theta^{r,i+1}_{K} = \theta^{r,i}_{K} - \eta_\textrm{MID}\nabla_{\theta^{r,i}_{K}} \mathcal{L}_{\textrm{MID}} \notag \\
            %     &\theta^{r,i+1}_\mathrm{gtc} = \theta^{r,i}_\mathrm{gtc} - \eta_\textrm{MID}\nabla_{\theta^{r,i}_\mathrm{gtc}} \mathcal{L}_{\textrm{MID}}. \notag 
            % \end{align}        
        }   
        Upload local model $\theta^{(r,E)}_{u}$ to the server. 
    }
    /*\textit{ Global Aggregation }*/ \\
    Compute global model $\theta^{(r)}_{g} = \sum_{u\in\gU}\theta^{(r,E)}_{u}$.
}
\end{algorithm}
In our Reptile-based L2U, the \emph{L2U Generator} and \emph{MID} are trained alternately in a continuous manner. At the start of each training round $r$, we begin with the \emph{L2U Generator} to prioritize the training of meaningful representations that can be faithfully reconstructed. Specifically, the L2U model is updated within the \emph{L2U Generator} as follows:
\begin{align}
    &\theta^{r,i+1}_{E} = \theta^{r,i}_{E} - \eta_\textrm{L2U}\nabla_{\theta^{r,i}_{E}} \mathcal{L}_{\textrm{L2U}}, \notag \\
    &\theta^{r,i+1}_{V} = \theta^{r,i}_{V} - \eta_\textrm{L2U}\nabla_{\theta^{r,i}_{V}} \mathcal{L}_{\textrm{L2U}}, \notag \\
    &\theta^{r,i+1}_{K} = \theta^{r,i}_{K} - \eta_\textrm{L2U}\nabla_{\theta^{r,i}_{K}} \mathcal{L}_{\textrm{L2U}}, \notag \\
    &\theta^{r,i+1}_2 = \theta^{r,i}_2 - \eta_\textrm{L2U}\nabla_{\theta^{r,i}_2} \mathcal{L}_{\textrm{L2U}}. 
\label{eq:L2U-update}
\end{align}
In the subsequent \emph{MID}, we focus on the training of featurizers $\theta_E, \theta_K$ and finetune the label classifier $\theta_\textrm{gtc}$ as follows: 
\begin{align}
    &\theta^{r,i+1}_{E} = \theta^{r,i}_{E} - \eta_\textrm{MID}\nabla_{\theta^{r,i}_{V}} \mathcal{L}^{r,i}_{\textrm{MID}}. \notag \\
    &\xi^{r,i+1} = \xi^{r,i} - \eta_\textrm{MID}\nabla_{\theta^{r,i}_\mathrm{gtc}} \mathcal{L}^{r,i}_{\textrm{MID}}.
\label{eq:MID-update}
\end{align}   
The details of the training process are shown in Algorithm~\ref{alg:L2U}.

%% file: sec/b1_appendix.tex
\onecolumn
\section{Client-level Federated Unlearning Derivation}\label{sec:FOUL-theorem}
% \begin{subequations}
% \label{eq:GIP-C-Pareto} 
% \begin{alignat} {3}
%     &   &	&  ~~~~\theta^{(r)}_{g} 
%     = \theta^{(r-1)}_{g} - \eta_g \rvg^{(r)}_\textrm{UL}
%     = \theta^{(r-1)}_{g} - \eta_g \Gamma_\textrm{UL} \rvg^{(r)}, 
%     \\
%     & \text{s.t.}
%     &	& ~~~~ \Gamma_\textrm{UL} = \argmax_{\Gamma}\min_{u\in \gU_\gR, v\in \gU_\gF}
%     \Big[
%       \langle\Gamma\rvg^{(r)},\rvg^{(r)}_u\rangle
%       - \langle \Gamma\rvg^{(r)},\rvg^{(r)}_v\rangle
%     - \gamma\Big(\Vert\Gamma\rvg^{(r)} - \rvg^{(r)}_\textrm{FL}\Vert^2 - \kappa\Vert\rvg^{(r)}_\textrm{FL}\Vert^2\Big)\Big].
% \end{alignat}
% \end{subequations}
% However, solving \eqref{eq:GIP-C-Pareto} is complex due to the min-max problems with two variables. Thus, we simplify the optimization problem \eqref{eq:GIP-C-Pareto} as follows:
\textbf{Theorem 1 (FOUL solution)}
    Given $\Gamma = \{\gamma^{(r)}_{u}\vert u\in\gU, \sum_{u\in \gU}\gamma^{(r)}_u = 1\}$ is the set of learnable coefficients at each round $r$. Invariant gradient direction $g^{(r)}_\textrm{UL}$ is characterized as follows:
    \begin{align}
        &g^{(r)}_\textrm{UL} = g^{(r)}_{\textrm{FL}} + \frac{\kappa\Vert g^{(r)}_{\textrm{FL}}\Vert}{\Vert g^{(r)}_{\Gamma_\gR} - g^{(r)}_{\Gamma_\gF}\Vert}(g^{(r)}_{\Gamma_\gR} - g^{(r)}_{\Gamma_\gF}), \\ \textrm{s.t.}\quad 
        &\Gamma_\gR^*, \Gamma_\gF^* = \arg\min_{\Gamma_\gR, \Gamma_\gF} (g^{(r)}_{\Gamma_\gR} - g^{(r)}_{\Gamma_\gF})\cdot g^{(r)}_{\textrm{FL}} + \sqrt{\kappa}\Vert g^{(r)}_{\textrm{FL}}\Vert\Vert g^{(r)}_{\Gamma_\gR} - g^{(r)}_{\Gamma_\gF}\Vert. \notag
    \end{align}
    % We denote $\Gamma^*$ as the optimal parameter set at round $r$.
\newline
\textbf{\textit{Proof.}} We consider $x = g^{(r)}_\textrm{UL}$ and consider the maximization problem with $x$ as optimization variable. 
Denote $\phi = \kappa^2 \Vert g^{(r)}_{\textrm{FL}}\Vert^2$. Note that $\min_u\langle g^{(r)}_u, g^{(r)}_{\textrm{FL}}\rangle = \min_{\gamma} \langle \sum_u \gamma_u h^{(r)}_u, h^{(r)}_g \rangle$. The Lagrangian of the objective is
\begin{align}
    \max_{x} \min_{\lambda, \gamma} 
    ~ (\sum_{u\in\gU_\gR} \gamma_u h^{(r)}_u)^{\top} x 
    - \beta(\sum_{v\in\gU_\gF} \gamma_v h^{(r)}_v)^{\top} x
    - \frac{\lambda}{2}\Vert g^{(r)}_{\textrm{FL}} - x\Vert^2 + \frac{\lambda}{2} \phi, \textrm{ s.t. } \lambda\geq 0.
\end{align}
Since the problem is a convex programming and and Slater's condition is satisfied for $\kappa>0$ (meanwhile, if $\kappa=0$, it can be easily verified that all results hold trivially), the strong duality holds. Consequently, the order of the $\min$ and $\max$ operations can be interchanged. For instance,
\begin{align}
    \min_{\lambda, \gamma} \max_{x} 
    ~ \underbrace{
      (\sum_{u\in\gU_\gR} \gamma_u h^{(r)}_u)^{\top} x 
    - \beta(\sum_{v\in\gU_\gF} \gamma_v h^{(r)}_v)^{\top} x
    - \frac{\lambda}{2}\Vert g^{(r)}_{\textrm{FL}} - x\Vert^2 + \frac{\lambda}{2} \phi}_{A_1}, \textrm{ s.t. } \lambda\geq 0.
\label{eq:min-max-pareto}
\end{align}
Taking $A_1$ into consideration. If we consider $\lambda, \gamma$ as constant, $x$ is the variable, $x$ achieves the optimal solution when $\partial A_1/\partial x = 0$. Thus, we have the followings:
\begin{align}
    \frac{\partial A_1}{\partial x} 
    = \sum_{u\in\gU_\gR} \gamma_u h^{(r)}_u
    - \beta\sum_{v\in\gU_\gF} \gamma_v h^{(r)}_v
    - \lambda (x - g^{(r)}_{\textrm{FL}}) = 0.
\end{align}
Another speaking, we have 
\begin{align}
    x = g^{(r)}_{\textrm{FL}} 
    + \Big(\sum_{u\in\gU_\gR} \gamma_u h^{(r)}_u 
    - \beta\sum_{v\in\gU_\gF} \gamma_v h^{(r)}_v\Big)/\lambda. 
\label{eq:optimal-gradient-solution}
\end{align}
Therefore, we have the followings:
\begin{align}
    A_1 
    % &= (\sum_{u\in\gU_\gR} \gamma_u h^{(r)}_u)^{\top} \Big(g^{(r)}_{\textrm{FL}} + \Big(\sum_{u\in\gU_\gR} \gamma_u h^{(r)}_u - \sum_{v\in\gU_\gF} \gamma_v h^{(r)}_v\Big)/\lambda\Big) \notag\\
    % &- (\sum_{v\in\gU_\gF} \gamma_v h^{(r)}_u)^{\top} \Big(g^{(r)}_{\textrm{FL}} + \Big(\sum_{u\in\gU_\gR} \gamma_u h^{(r)}_u - \sum_{v\in\gU_\gF} \gamma_v h^{(r)}_v\Big)/\lambda\Big) \notag\\
    % &- \frac{\lambda}{2}\Vert g^{(r)}_{\textrm{FL}} - \Big(g^{(r)}_{\textrm{FL}} + \Big(\sum_{u\in\gU_\gR} \gamma_u h^{(r)}_u - \sum_{v\in\gU_\gF} \gamma_v h^{(r)}_v\Big)/\lambda\Big)\Vert^2 + \frac{\lambda}{2} \phi.
    &= (\sum_{u\in\gU_\gR} \gamma_u h^{(r)}_u)^{\top} \Big(g^{(r)}_{\textrm{FL}} + \Big(\sum_{u\in\gU_\gR} \gamma_u h^{(r)}_u 
    - \beta\sum_{v\in\gU_\gF} \gamma_v h^{(r)}_v\Big)/\lambda\Big) \notag\\
    &- \beta (\sum_{v\in\gU_\gF} \gamma_v h^{(r)}_u)^{\top} \Big(g^{(r)}_{\textrm{FL}} + \Big(\sum_{u\in\gU_\gR} \gamma_u h^{(r)}_u 
    - \beta\sum_{v\in\gU_\gF} \gamma_v h^{(r)}_v\Big)/\lambda\Big) 
     - \frac{1}{2\lambda}\Big\Vert\sum_{u\in\gU_\gR} \gamma_u h^{(r)}_u 
     - \beta\sum_{v\in\gU_\gF} \gamma_v h^{(r)}_v\Big\Vert^2 + \frac{\lambda}{2} \phi.
    % \notag \\
    % &= (\sum^{U}_{u=1} \gamma_u h^{(r)}_u)^{\top} \Big(g^{(r)}_{\textrm{FL}} + \Big(\sum^{U}_{u=1} \gamma_u h^{(r)}_u\Big)/\lambda\Big) - \frac{\lambda}{2}\Vert \frac{1}{\lambda}\sum^{U}_{u=1} \gamma_u h^{(r)}_u\Vert^2 + \frac{\lambda}{2} \phi.
\label{eq:min-max-pareto=2}
\end{align}
Substituting $g^{(r)}_{\Gamma_\gR} = \sum_{u\in\gU_\gR} \gamma_u h^{(r)}_u$ and $g^{(r)}_{\Gamma_\gF} = \sum_{v\in\gU_\gR} \gamma_v h^{(r)}_v$. Consider the optimization problem of \eqref{eq:min-max-pareto=2}, we have: 
\begin{align}
    A_1 
    &= g^{(r)\top}_{\Gamma_\gR} \Big(g^{(r)}_{\textrm{FL}} + (g^{(r)}_{\Gamma_\gR} - \beta g^{(r)}_{\Gamma_\gF})/\lambda\Big) 
     - \beta g^{(r)\top}_{\Gamma_\gF} \Big(g^{(r)}_{\textrm{FL}} + (g^{(r)}_{\Gamma_\gR} - \beta g^{(r)}_{\Gamma_\gF})/ \lambda\Big) 
     - \frac{1}{2\lambda}\Vert g^{(r)}_{\Gamma_\gR} - \beta g^{(r)}_{\Gamma_\gF} \Vert^2 
     + \frac{\lambda}{2} \phi \notag \\
    &= (g^{(r)}_{\Gamma_\gR} - \beta g^{(r)}_{\Gamma_\gF})^{\top} g^{(r)}_{\textrm{FL}} 
     + \frac{1}{\lambda} (g^{(r)}_{\Gamma_\gR} - \beta g^{(r)}_{\Gamma_\gF})^{\top} (g^{(r)}_{\Gamma_\gR} - \beta g^{(r)}_{\Gamma_\gF}) 
     - \frac{1}{2\lambda}\Vert g^{(r)}_{\Gamma_\gR} - \beta g^{(r)}_{\Gamma_\gF} \Vert^2 
     + \frac{\lambda}{2} \phi \notag \\
    &= (g^{(r)}_{\Gamma_\gR} - \beta g^{(r)}_{\Gamma_\gF})^{\top} g^{(r)}_{\textrm{FL}} + \frac{1}{2\lambda} \Vert g^{(r)}_{\Gamma_\gR} 
    - \beta g^{(r)}_{\Gamma_\gF} \Vert^2 + \frac{\lambda}{2} \phi.
\end{align}
Therefore, we have Eq.~\ref{eq:min-max-pareto} is equivalent to 
\begin{align}
     &\min_{\lambda, \gamma} 
     ~ \underbrace{(g^{(r)}_{\Gamma_\gR} 
     - \beta g^{(r)}_{\Gamma_\gF})^{\top} g^{(r)}_{\textrm{FL}} 
     + \frac{1}{2\lambda} \Vert g^{(r)}_{\Gamma_\gR} 
     - \beta g^{(r)}_{\Gamma_\gF} \Vert^2 + \frac{\lambda}{2} \phi}_{A_2}.
\label{eq:equivalent-min-max}
\end{align}
Next, we consider $\lambda$ as variable to find the optimal value. Subsequently, optimization problem \eqref{eq:equivalent-min-max} is equivalent to the following relationship: 
\begin{align}
    \frac{\partial}{\partial \lambda} A_2 = 
    - \frac{1}{2\lambda^2} \Vert g^{(r)}_{\Gamma_\gR} 
    - \beta g^{(r)}_{\Gamma_\gF} \Vert^2 + \frac{1}{2} \phi = 0.
\end{align}
Therefore, the equation achieves the optimality as $\lambda = \Vert g^{(r)}_{\Gamma_\gR} - g^{(r)}_{\Gamma_\gF} \Vert / \phi^{1/2}$. Combining with \eqref{eq:equivalent-min-max} and \eqref{eq:optimal-gradient-solution}, we have the followings:
\begin{align}
    &g^{(r)}_\textrm{UL} 
    = g^{(r)}_{\textrm{FL}} 
    + \frac{\kappa\Vert g^{(r)}_{\textrm{FL}}\Vert}{\Vert g^{(r)}_{\Gamma_\gR} 
    - \beta g^{(r)}_{\Gamma_\gF}\Vert}
    (g^{(r)}_{\Gamma_\gR} - \beta g^{(r)}_{\Gamma_\gF}), \\ \textrm{s.t.}\quad 
    &\Gamma_\gR^*, \Gamma_\gF^* 
    = \arg\min_{\Gamma_\gR, \Gamma_\gF} (g^{(r)}_{\Gamma_\gR} 
    - \beta g^{(r)}_{\Gamma_\gF})\cdot g^{(r)}_{\textrm{FL}} + \sqrt{\kappa}\Vert g^{(r)}_{\textrm{FL}}\Vert 
    \Vert g^{(r)}_{\Gamma_\gR} - \beta g^{(r)}_{\Gamma_\gF}\Vert. \notag
\end{align}
This solves the problem.

%% file: sec/b3_appendix.tex
\clearpage
\section{Pseudo Algorithm for unlearning stage of FOUL}\label{app:FOUL-pseudo}
\begin{algorithm}
\caption{Unlearning Stage via On-server Matching Gradient}
\label{alg:FedOMG}
    \KwIn{set of retain users $\gU_\gR$, forget users $\gU_\gF$, $\gU=\{\gU_\gR, \gU_\gF\}$, number of communication rounds $R$, local learning rate $\eta$, global learning rate $\eta_g$, searching space hyper-parameter $\kappa$.}
    \KwOut{$\theta_{V,g}^{(R)}$}

    \textbf{Clients Update:}\\
    \For{client $u\in\gU$}{
        \textbf{Receive} global non-causal encoder $\theta_{V,u}^{(r)} = \theta_{V,g}^{(r)}$\;
        Merge global non-causal encoder $\theta_{V,u}^{(r)}$ into local model $\theta_{u}^{(r)}$\;
        \For{local epoch $e \in E$}{
                Sample mini-batch $\zeta$ from local data $\mathcal{D}_u$\;
                Calculate gradient $\nabla_{\theta^{(r,e)}_{V,u}} \gE(\theta^{(r,e)}_u, \zeta)$\;
                Update client's non-causal encoder: 
                $\theta^{(r,e+1)}_{V,u} = \theta^{(r,e)}_{V,u} - \eta \nabla_{\theta^{(r,e)}_{V,u}} \gE(\theta^{(r,e)}_u, \zeta)$\;
            }
            Upload client's non-causal encoder $\theta^{(r,E)}_{V,u}$ to server.
        % \For{round $r=0,\ldots, R$}{
        %     \For{local epoch $e \in E$}{
        %         Sample mini-batch $\zeta$ from local data $\mathcal{D}_u$\;
        %         Calculate gradient $\nabla \gE(\theta^{(r,e)}_u, \zeta)$\;
        %         Update client's model: 
        %         $\theta^{(r,e+1)}_u = \theta^{(r,e)}_u - \eta \nabla \gE(\theta^{(r,e)}_u, \zeta)$\;
        %     }
        %     Upload client's model $\theta^{(r,E)}_u$ to server\;
        % }
    }
    \textbf{Server Optimization:}\\
    \For{round $r=0,\ldots, R$}{
        \textbf{Clients Updates}\;
        Calculate $g^{(r)}_{u} = \theta^{(r,E)}_{V,u} - \theta^{(r)}_{V,u}$, ~~ $\rvg^{(r)} = \{g^{(r)}_{u} \vert u\in\gU\}$\;
        Calculate $g^{(r)}_{FL}$ (e.g., $g^{(r)}_{FL} = \frac{1}{U} \sum_{u=1}^{U} g^{(r)}_{u}$ as the FedAvg update)\;
        Calculate $g^{(r)}_{\Gamma_\gR} = \sum_{u\in\gU_\gR} \gamma_u g^{(r)}_{u}$ as the retain user set update)\;
        Calculate $g^{(r)}_{\Gamma_\gF} = \sum_{v\in\gU_\gF} \gamma_v g^{(r)}_{v}$ as the forget user set update)\;
        Solve for $\Gamma^* = \{\Gamma_\gR^*, \Gamma_\gF^*\}$, where $\Gamma_\gR=\{\gamma_u\vert u\in\gU_\gR\}$, $\Gamma_\gF=\{\gamma_v\vert v\in\gU_\gF\}$:
        \[
            \Gamma_\gR^*, \Gamma_\gF^* = \arg\min_{\Gamma_\gR, \Gamma_\gF} (g^{(r)}_{\Gamma_\gR} - g^{(r)}_{\Gamma_\gF})\cdot g^{(r)}_{\textrm{FL}} + \sqrt{\kappa}\Vert g^{(r)}_{\textrm{FL}}\Vert\Vert g^{(r)}_{\Gamma_\gR} - g^{(r)}_{\Gamma_\gF}\Vert.
        \]\\
        % where $g^{(r)}_{\pi} = \sum_{i=1}^{K}\pi^{(r)}_i g^{(r)}_i$, subject to $\sum_{i=1}^{K} \pi^{(r)}_i = 1$\;
        Update unlearn gradient: $g^{(r)}_\textrm{FOUL} = g^{(r)}_{\textrm{FL}} + \frac{\kappa\Vert g^{(r)}_{\textrm{FL}}\Vert}{\Vert g^{(r)}_{\Gamma_\gR} - g^{(r)}_{\Gamma_\gF}\Vert}(g^{(r)}_{\Gamma_\gR} - g^{(r)}_{\Gamma_\gF}).$\\
        Update the model: $\theta_{V,g}^{(r)} = \theta_{V,g}^{(r-1)} - \eta_g g^{(r)}_\textrm{FOUL}.$
    }
\end{algorithm}

%% file: sec/b4_appendix.tex
\clearpage
\section{Detailed Results}\label{app:detailed-results}

\begin{table*}[htbp]\small
\centering
\caption{Comparison of methods on ResNet-18 with VLCS and OfficeHome datasets. The table reports Forget Accuracy (FA), Remaining Accuracy (RA), Testing Accuracy (TA), and Membership Inference Attack (MIA), with values in parentheses showing the difference from the Retrain baseline. For the FOUL method, (1) corresponds to the L2U phase, and (2) refers to the on-server gradient matching unlearning phase. Note that we \textbf{bold} results achieving the closest TA accuracy to Retrain.}
\label{tab:results_vlcs_oh}
\resizebox{\textwidth}{!}{
\begin{tabular}{@{}l@{\;}l@{\;}c@{\;}c@{\;}c@{\;}c@{\quad}c@{\;}c@{\;}c@{\;}c@{}}
\toprule
\multirow{2}{*}{Group} & \multirow{2}{*}{Methods} & \multicolumn{4}{c}{VLCS} & \multicolumn{4}{c}{OfficeHome} \\
\cmidrule(lr){3-6} \cmidrule(lr){7-10}
 & & FA $(\downarrow)$ & RA $(\uparrow)$ & TA $(\uparrow)$ & MIA $(\downarrow)$ &
 FA $(\downarrow)$ & RA $(\uparrow)$ & TA $(\uparrow)$ & MIA $(\downarrow)$ \\
\midrule
\multirow{3}{*}{Retrain} & Retrain & $63.14$ & $73.88$ & $70.16$ & $52.96$ & $45.90$ & $66.32$ & $59.08$ & $47.64$ \\
& FATS & $66.51$ \textcolor{blue}{$(+3.37)$} & $72.09$ \textcolor{blue}{$(-1.79)$} & $68.62$ \textcolor{blue}{$(-1.54)$} & $58.68$ \textcolor{blue}{$(+5.72)$} &
$48.73$ \textcolor{blue}{$(+2.83)$} & $64.05$ \textcolor{blue}{$(-2.27)$} & $\mathbf{58.13}$ \textcolor{blue}{$(-0.95)$} & $57.46$ \textcolor{blue}{$(+9.82)$} \\
& NoT & $65.52$ \textcolor{blue}{$(+2.38)$} & $71.04$ \textcolor{blue}{$(-2.84)$} & $67.83$ \textcolor{blue}{$(-2.33)$} & $61.93$ \textcolor{blue}{$(+8.97)$} &
$47.57$ \textcolor{blue}{$(+1.67)$} & $61.51$ \textcolor{blue}{$(-4.81)$} & $56.15$ \textcolor{blue}{$(-2.93)$} & $59.79$ \textcolor{blue}{$(+12.15)$} \\
\cmidrule(lr){1-10}
\multirow{6}{*}{\shortstack{Appro.\\Unlearn.}} & FedCDP & $69.78$ \textcolor{blue}{$(+6.64)$} & $69.95$ \textcolor{blue}{$(-3.93)$} & $68.63$ \textcolor{blue}{$(-1.53)$} & $74.56$ \textcolor{blue}{$(+21.60)$} &
$50.07$ \textcolor{blue}{$(+4.17)$} & $61.83$ \textcolor{blue}{$(-4.49)$} & $57.73$ \textcolor{blue}{$(-1.35)$} & $78.06$ \textcolor{blue}{$(+30.42)$} \\
& FedRecovery & $69.33$ \textcolor{blue}{$(+6.19)$} & $68.79$ \textcolor{blue}{$(-5.09)$} & $67.61$ \textcolor{blue}{$(-2.55)$} & $77.34$ \textcolor{blue}{$(+24.38)$} &
$47.65$ \textcolor{blue}{$(+1.75)$} & $60.81$ \textcolor{blue}{$(-5.51)$} & $56.35$ \textcolor{blue}{$(-2.73)$} & $80.41$ \textcolor{blue}{$(+32.77)$} \\
& FedOSD & $65.74$ \textcolor{blue}{$(+2.60)$} & $71.85$ \textcolor{blue}{$(-2.03)$} & $67.98$ \textcolor{blue}{$(-2.18)$} & $61.52$ \textcolor{blue}{$(+8.56)$} &
$47.81$ \textcolor{blue}{$(+1.91)$} & $64.15$ \textcolor{blue}{$(-2.17)$} & $56.96$ \textcolor{blue}{$(-2.12)$} & $57.21$ \textcolor{blue}{$(+9.57)$} \\
& FFMU & $66.25$ \textcolor{blue}{$(+3.11)$} & $69.75$ \textcolor{blue}{$(-4.13)$} & $66.55$ \textcolor{blue}{$(-3.61)$} & $63.06$ \textcolor{blue}{$(+10.10)$} &
$49.72$ \textcolor{blue}{$(+3.82)$} & $62.66$ \textcolor{blue}{$(-3.66)$} & $56.05$ \textcolor{blue}{$(-3.03)$} & $62.79$ \textcolor{blue}{$(+15.15)$} \\
& FUSED & $68.08$ \textcolor{blue}{$(+4.94)$} & $70.91$ \textcolor{blue}{$(-2.97)$} & $\mathbf{69.81}$ \textcolor{blue}{$(-0.35)$} & $60.84$ \textcolor{blue}{$(+7.88)$} &
$49.41$ \textcolor{blue}{$(+3.51)$} & $62.55$ \textcolor{blue}{$(-3.77)$} & $57.55$ \textcolor{blue}{$(-1.53)$} & $58.72$ \textcolor{blue}{$(+11.08)$} \\
& MoDE & $64.93$ \textcolor{blue}{$(+1.79)$} & $70.72$ \textcolor{blue}{$(-3.16)$} & $68.19$ \textcolor{blue}{$(-1.97)$} & $62.28$ \textcolor{blue}{$(+9.32)$} &
$48.10$ \textcolor{blue}{$(+2.20)$} & $61.39$ \textcolor{blue}{$(-4.93)$} & $56.44$ \textcolor{blue}{$(-2.64)$} & $60.15$ \textcolor{blue}{$(+12.51)$} \\
\cmidrule(lr){1-10}
\multirow{2}{*}{FOUL} & (1) & $61.72$ \textcolor{blue}{$(-1.42)$} & $76.78$ \textcolor{blue}{$(+2.90)$} & $68.97$ \textcolor{blue}{$(-1.19)$} & $56.85$ \textcolor{blue}{$(+3.89)$} &
$43.69$ \textcolor{blue}{$(-2.21)$} & $67.03$ \textcolor{blue}{$(+0.71)$} & $61.96$ \textcolor{blue}{$(+2.88)$} & $56.53$ \textcolor{blue}{$(+8.89)$} \\
& (1) + (2) & $62.91$ \textcolor{blue}{$(-0.23)$} & $77.15$ \textcolor{blue}{$(+3.27)$} & $69.10$ \textcolor{blue}{$(-1.06)$} & $55.83$ \textcolor{blue}{$(+2.87)$} &
$44.17$ \textcolor{blue}{$(-1.73)$} & $65.78$ \textcolor{blue}{$(-0.54)$} & $60.77$ \textcolor{blue}{$(+1.69)$} & $57.48$ \textcolor{blue}{$(+9.84)$} \\
\bottomrule
\end{tabular}
}
\end{table*}

\clearpage
\section{Evaluations on convergence of learning phase}\label{app:coefficient}
As the joint loss function comprises multiple components, it is theoretically challenging to optimize. Therefore, we conduct ablation experiments on VLCS dataset by varying the coefficients of each component in the joint loss function, as illustrated in Figures~\ref{fig:learning-coeff}. In each experiment, we fix all other coefficients to $1$ while varying only one coefficient at a time. As shown in the figures, the model converges within fewer than $100$ training rounds, which is comparable to the convergence behavior of vanilla FL. Moreover, the low invariant and variance losses indicate that the learned disentangled representations closely align with the theoretical expectations, where the causal and non-causal representations exhibit invariant and variant properties, respectively.
\begin{figure}[!ht]
\centering
\includegraphics[width = 0.52\linewidth]{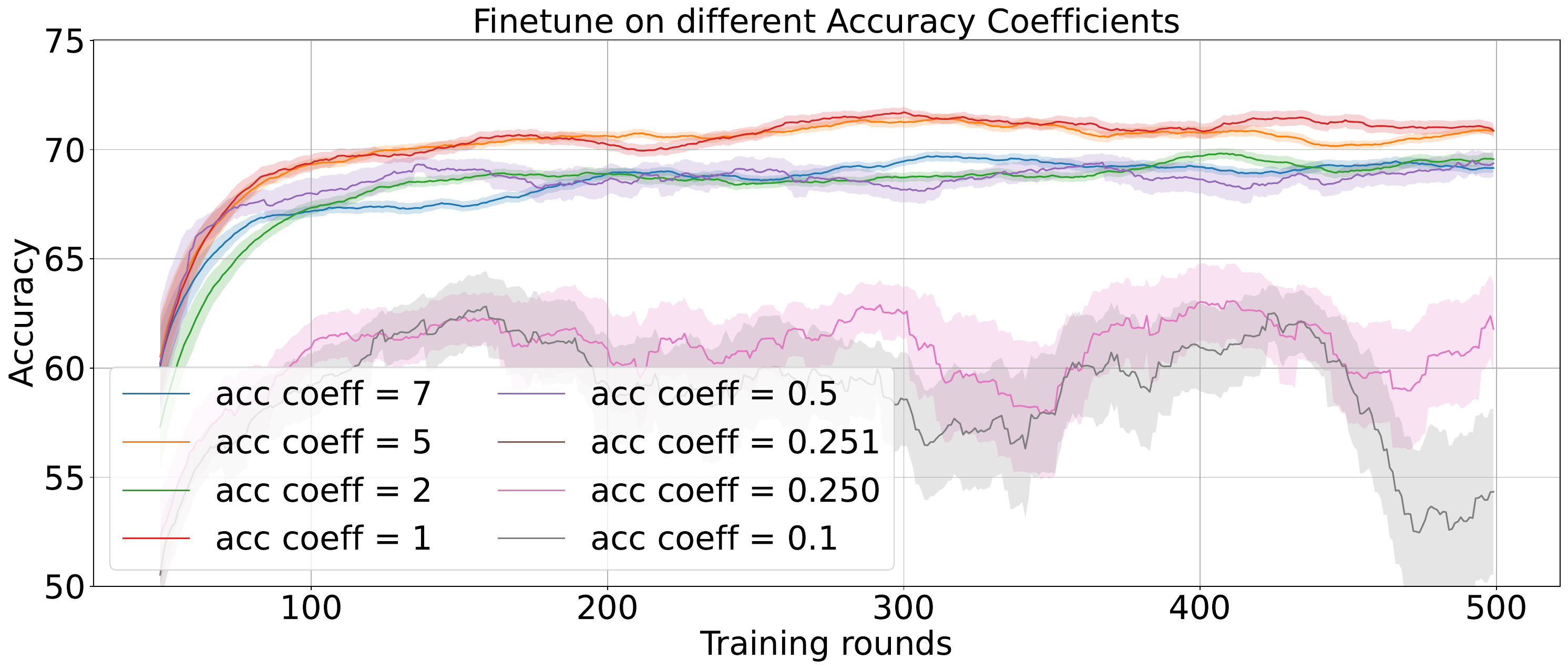}
\includegraphics[width = 0.23\linewidth]{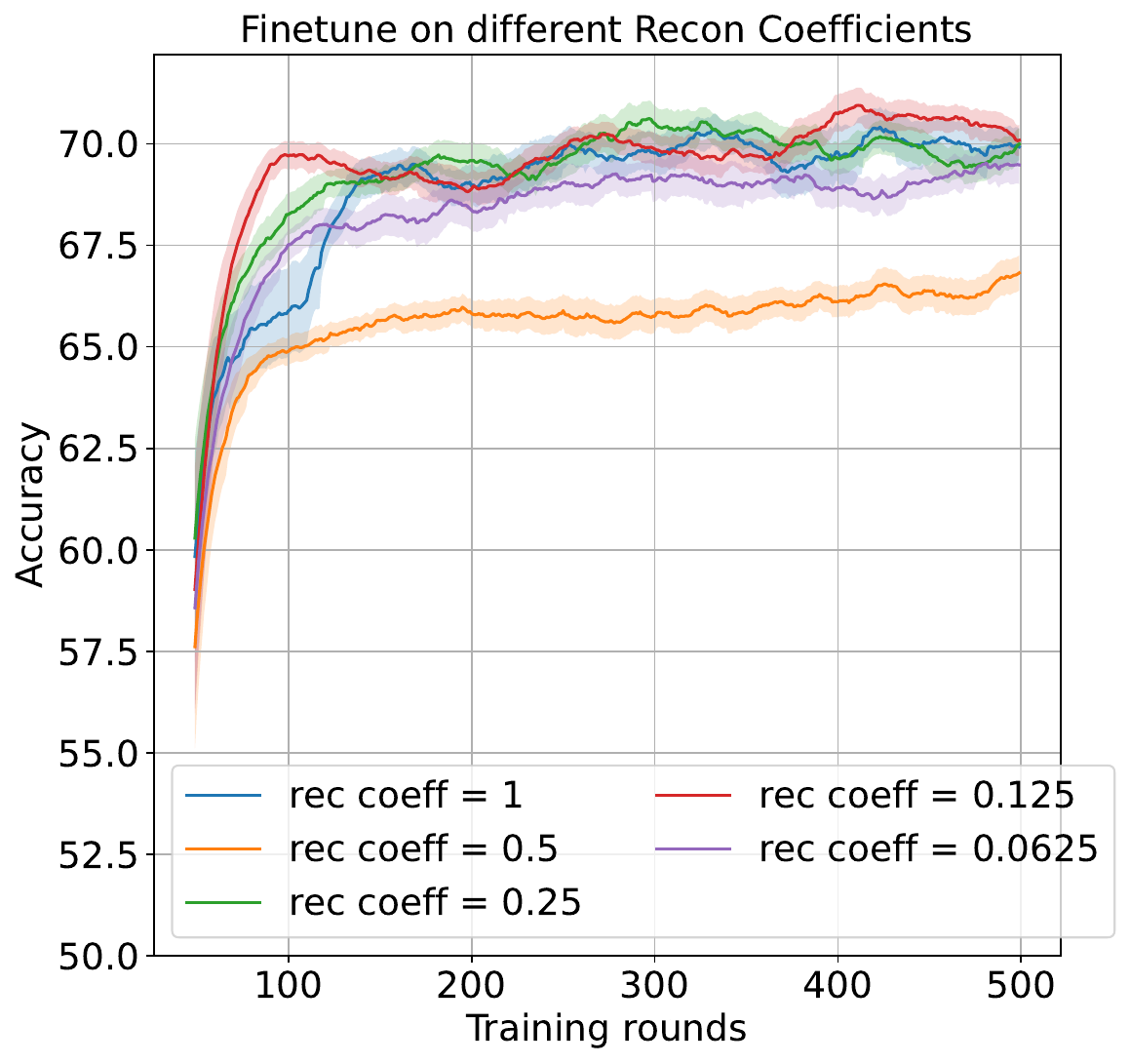} 
\includegraphics[width = 0.217\linewidth]{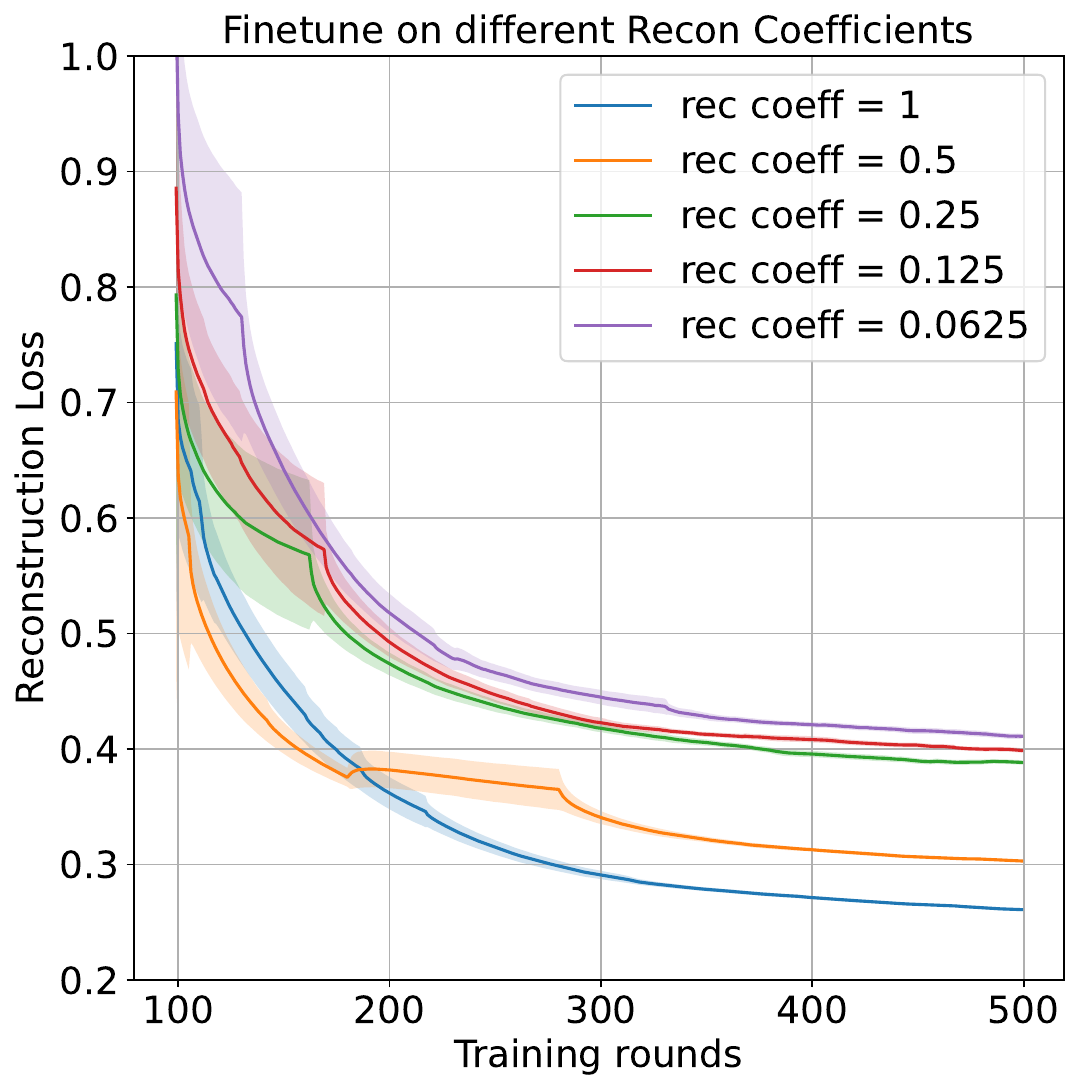} \\
\includegraphics[width = 0.235\linewidth]{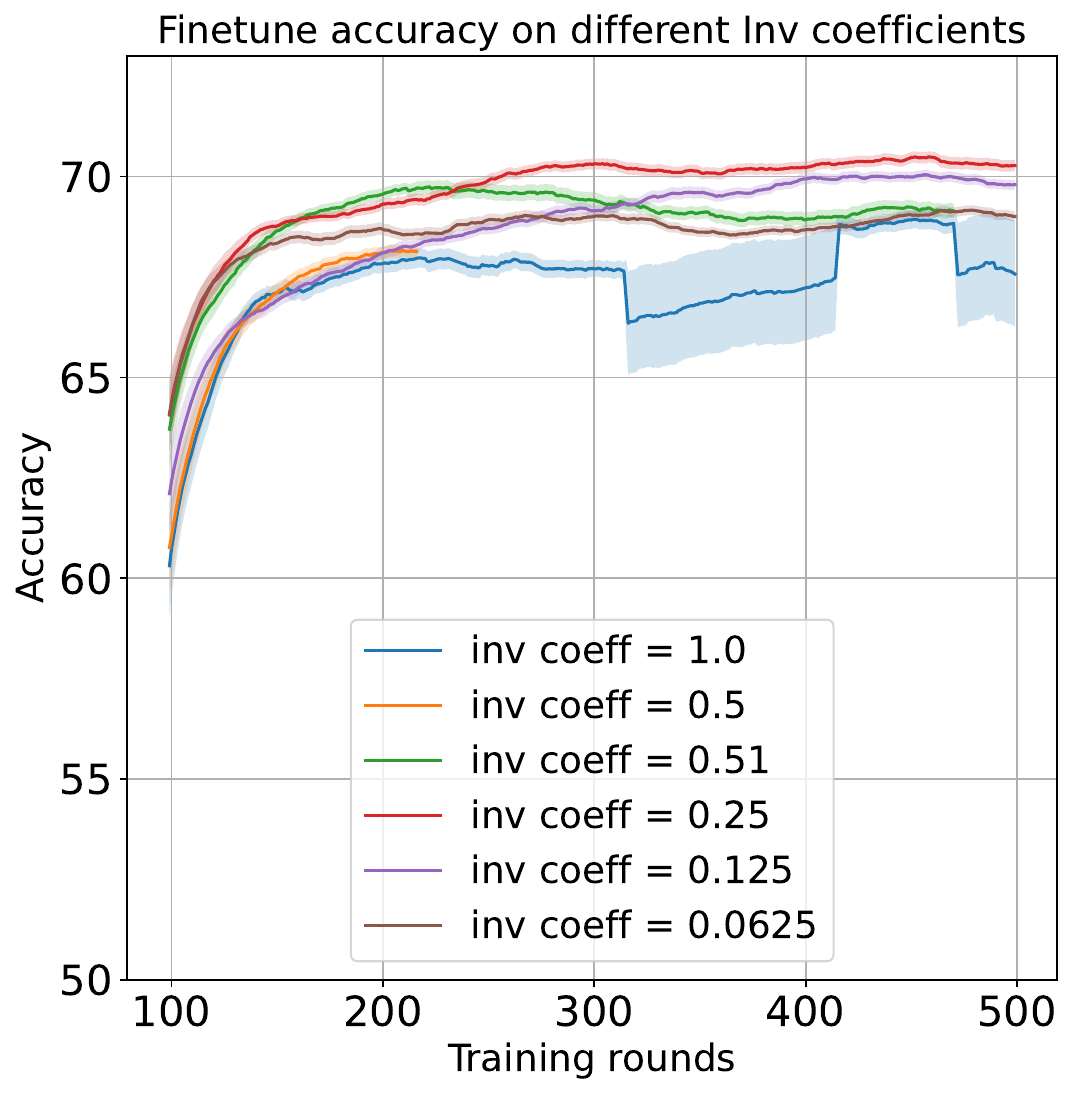} 
\includegraphics[width = 0.25\linewidth]{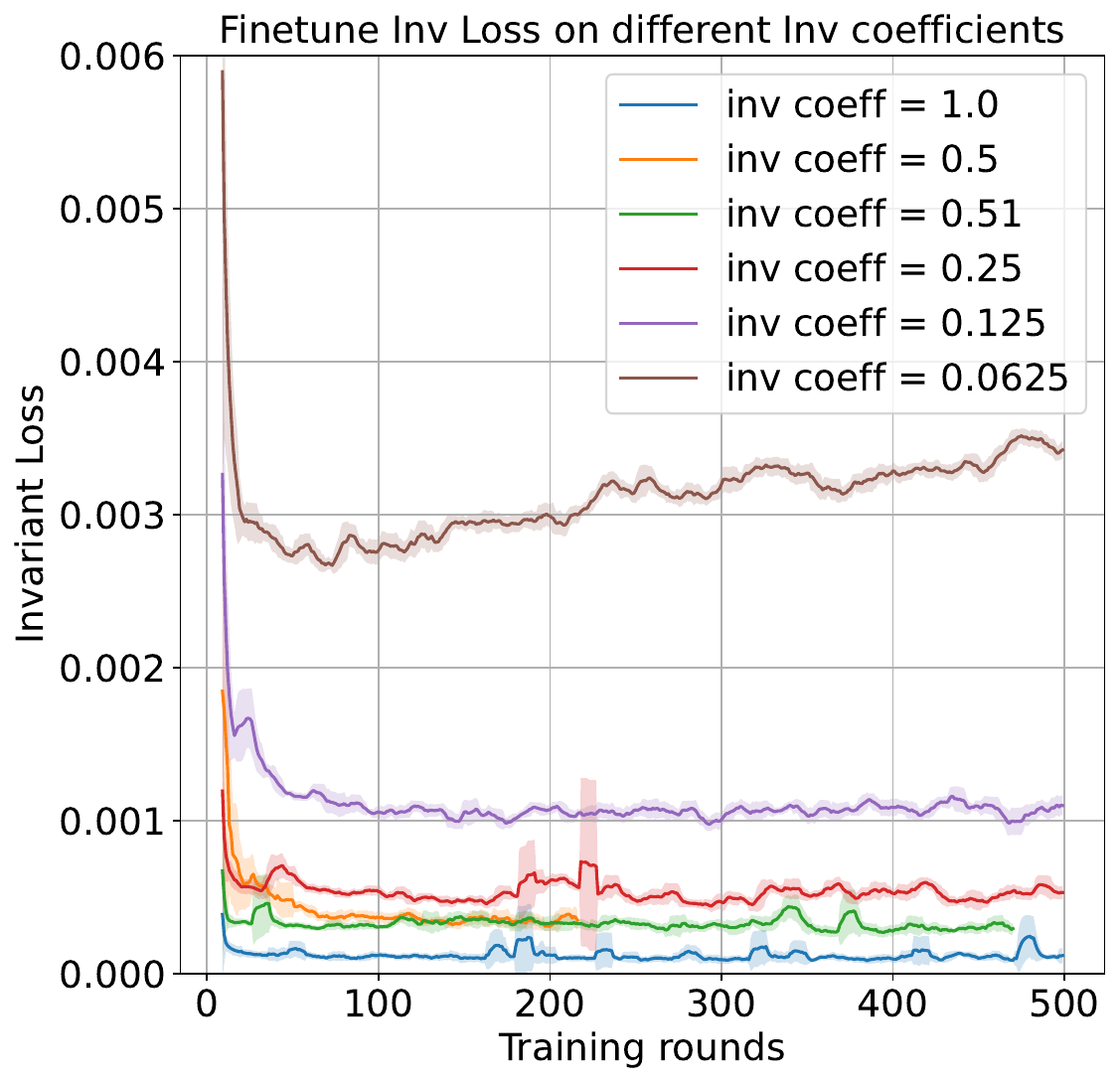} 
\includegraphics[width = 0.235\linewidth]{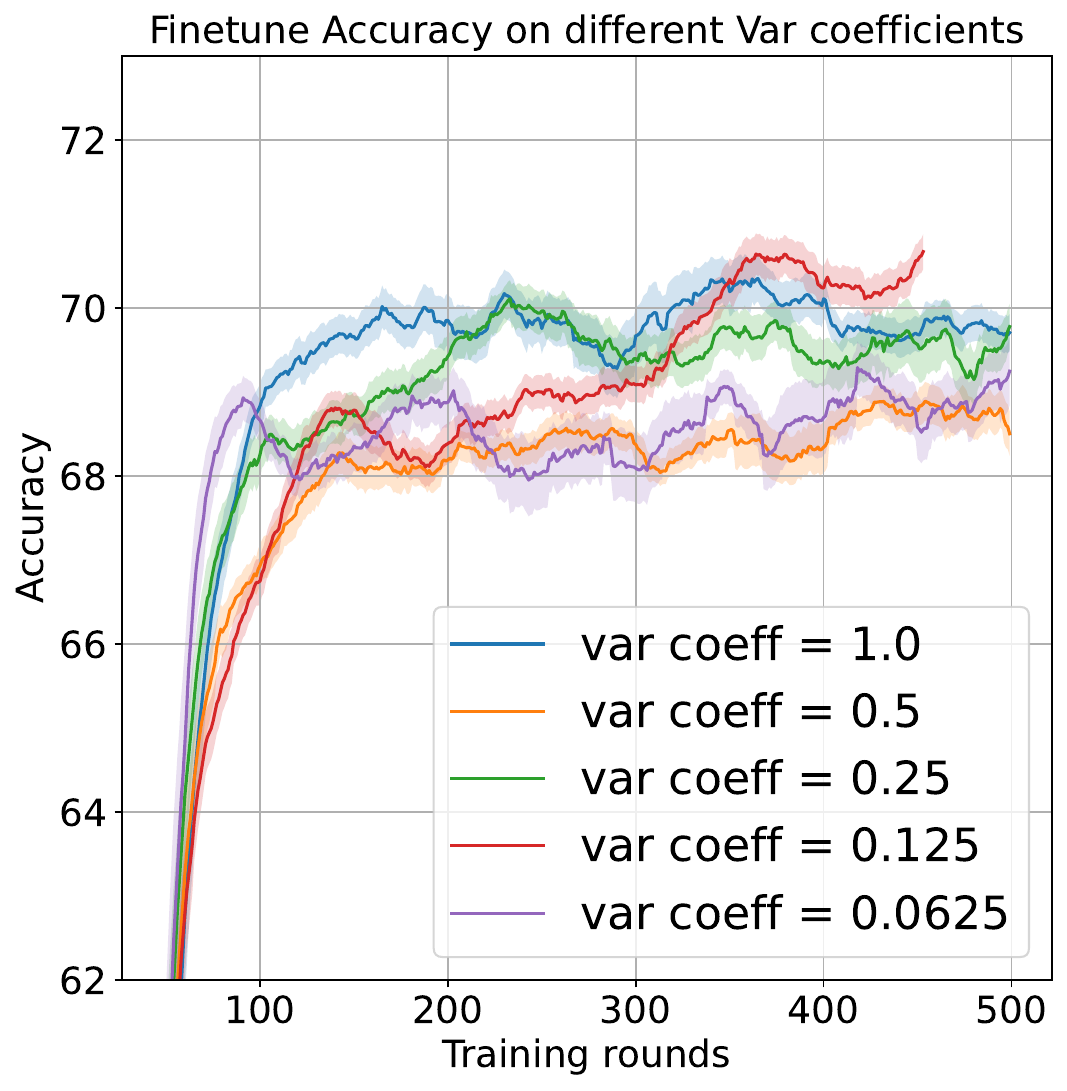} 
\includegraphics[width = 0.25\linewidth]{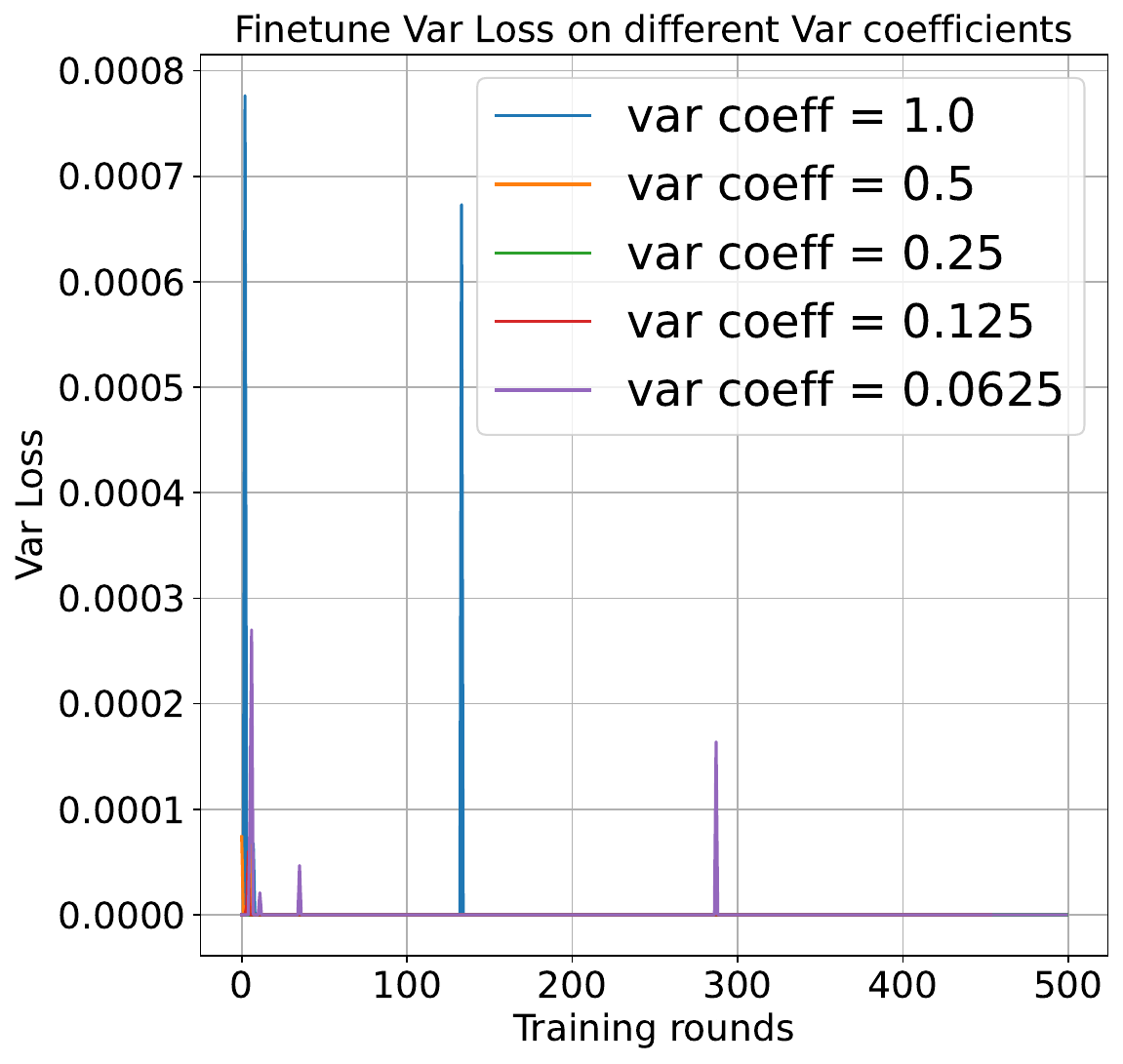} 
\caption{Learning phase of FOUL under different learning coefficient settings.}
\label{fig:learning-coeff}
\end{figure}

\clearpage
\section{Assessing the domain generalization capabilities of FOUL}\label{app:invariant}
We evaluate the performance of FOUL in extracting and predicting invariant representations to verify the robustness of causality-invariant features against domain generalization gaps. To this end, we employ the DomainBed library. Specifically, we train three distinct FOUL models on the Colored-MNIST dataset \cite{2020-DG-DomainBed}, each using a different domain for training, and subsequently evaluate them on all remaining domains. The results are summarized in Table~\ref{tab:FOUL-domain}.

Table~\ref{tab:FOUL-domain} highlights two key findings. First, the invariant representations $z_K$ capture substantial causal information, enabling accurate prediction of the ground-truth labels in the source domain, which is consistent with the underlying CSM. Second, these invariant representations remain consistent across different data domains. Consequently, by preserving the knowledge of the semantic featurizer during the unlearning phase, FOUL retains the domain-invariant knowledge, ensuring that unlearning on the forget set minimally impacts performance on the retain set.

\begin{table}[!ht]
\caption{Assessment of FOUL across different domains. The evaluation is conducted on the Colored-MNIST dataset, where each domain corresponds to images with varying degrees of correlation between color and label. Two aspects are assessed: (1) Accuracy: evaluating the model’s ability to predict labels in each domain using only the invariant representations; and (2) Invariance: measuring the disparity between the invariant representations $z^i_K$ generated from a target domain $i$ and those $z^j_K$ obtained from the training domain $j$, to quantify the consistency of invariant features across domains.}
\label{tab:FOUL-domain}
\centering
\begin{tabular}{lccc}
\hline
\textbf{Algorithm} & $\textbf{+90\%}$ & $\textbf{+80\%}$ & $\textbf{-90\%}$ \\ \hline \hline
\multicolumn{4}{c}{\textbf{Domain 1 (+90\%})}                                                                                             \\ \hline
Accuracy                            & 89.5 $\pm$ 0.1                  & 82.3 $\pm$ 0.5                  & 73.2 $\pm$ 0.3                  \\
Invariant                           & 0.0003 $\pm$ 10e-5              & 0.0002 $\pm$ 10e-5              & 0.0002 $\pm$ 10e-5              \\ \hline \hline
\multicolumn{4}{c}{\textbf{Domain 2 (+80\%})}                                                                                  \\ \hline
Accuracy                            & 73.9 $\pm$ 0.1                  & 82.3 $\pm$ 0.5                  & 73.2 $\pm$ 0.3                  \\
Invariant                           & 0.0003 $\pm$ 10e-5              & 0.0002 $\pm$ 10e-5              & 0.0002 $\pm$ 10e-5              \\ \hline \hline
\multicolumn{4}{c}{\textbf{Domain 3 (-90\%})}                                                                                  \\ \hline
Accuracy                            & 89.5 $\pm$ 0.1                  & 80.7 $\pm$ 0.5                  & 73.2 $\pm$ 0.3                  \\
Invariant                           & 0.0003 $\pm$ 10e-5              & 0.0002 $\pm$ 10e-5              & 0.0002 $\pm$ 10e-5              \\ \hline
\end{tabular}
\end{table}

\clearpage
\section{Assessment of Knowledge Extraction Capability}\label{app:knowledge-capability}
To determine the extent of invariant knowledge extraction from the data and the appropriate semantic knowledge size to represent the data from each label, we keep the size of the embedding layer constant (i.e., $C_\textrm{IB}=32$), while simultaneously varying the number of invariant channels and setting the variant channel size to $C_\textrm{v} = C_\textrm{IB} - C_\textrm{iv}$. This approach is adopted to ensure that the performance of data reconstruction remains unaffected by the informational bottleneck. 
The results evaluated  on VLCS dataset are shown in Fig.~\ref{fig:unbiased-knowledge}. 
As can be seen, the L2U stage performs optimally when the number of invariant knowledge channels are set to $C \geq 16$. Additionally, the data reconstruction remains consistent as the invariant knowledge size is increased. This indicates that higher data reconstruction efficiency can be achieved by using a larger invariant knowledge size (e.g., $C_\textrm{iv} \geq 24$, $C_\var \leq 8$). This also implies that a high compression ratio of $(3\times 32\times 32) : (8\times 64) = 6:1$ can be attained without quality loss in the variant data, which needs to be transmitted over the physical channel.
\begin{figure}[!ht]
\centering
\includegraphics[width = 0.475\linewidth]{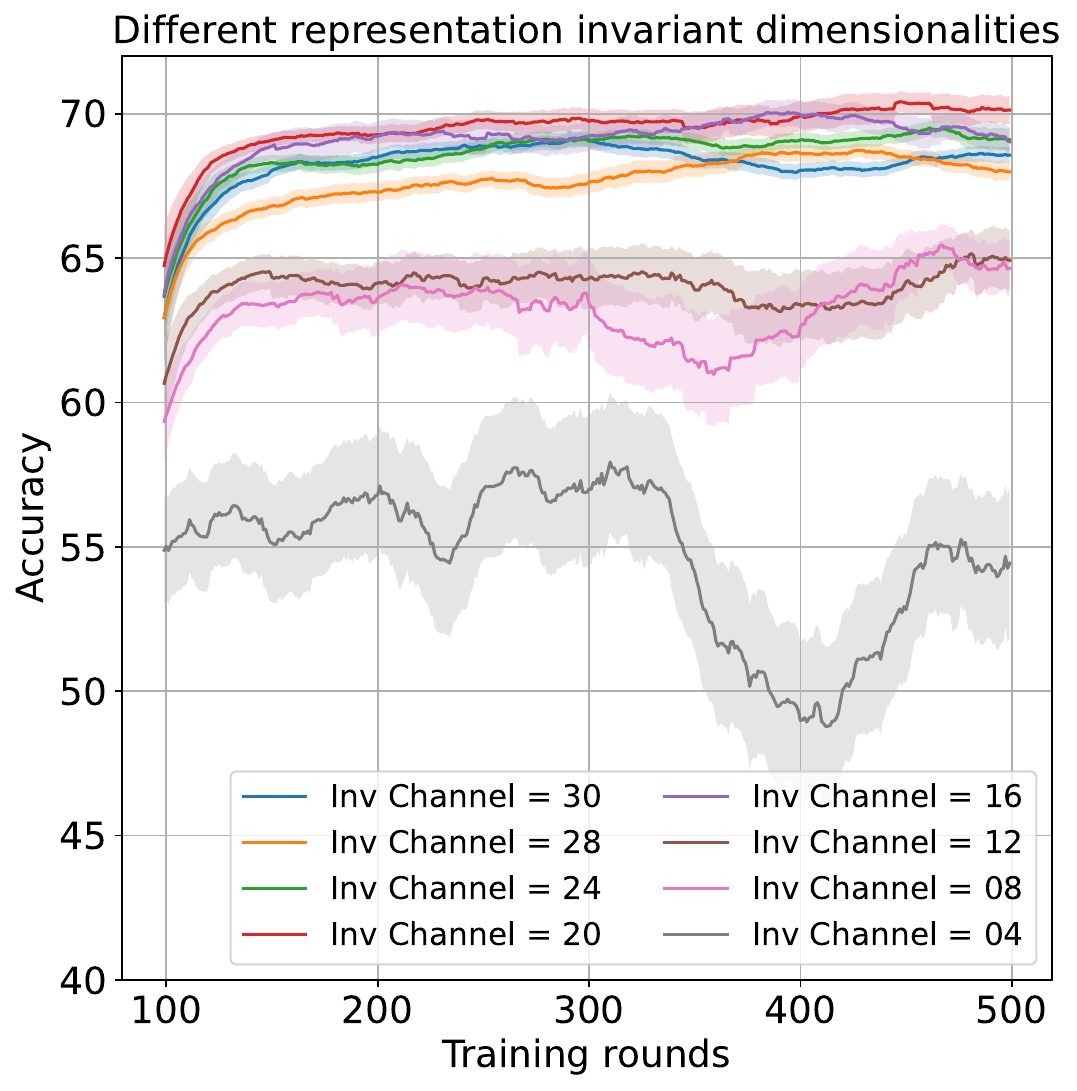}
\includegraphics[width = 0.51\linewidth]{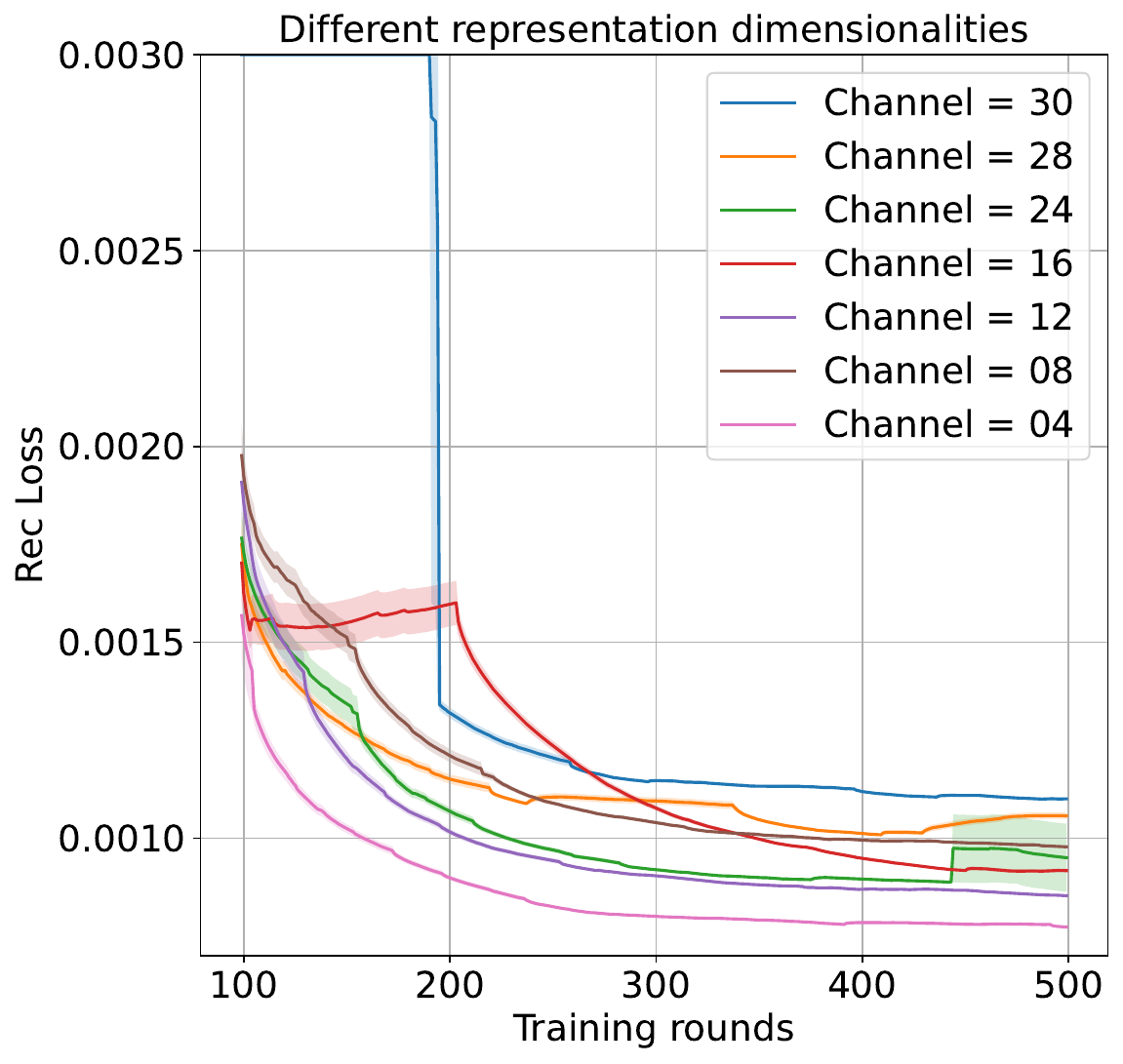}
\caption{Illustrations of the efficacy of invariant knowledge featurizer.}
\label{fig:unbiased-knowledge}
\end{figure}

\clearpage
\section{Ablation test on unlearning stage}\label{app:ablation-unlearning}
\begin{figure}[ht]
    \centering
    \subfloat[]{
        \includegraphics[width=0.4\linewidth]{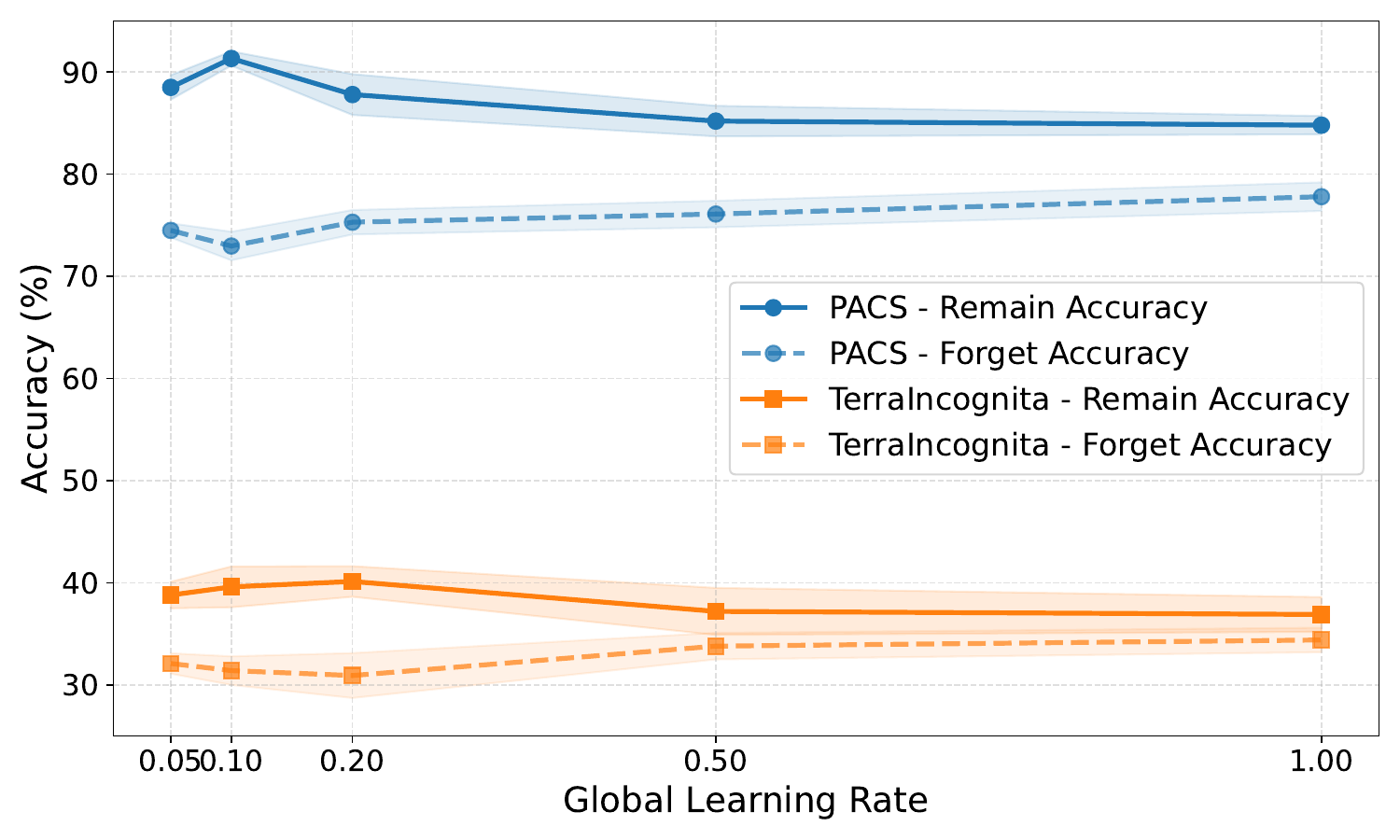}
        \label{fig:lr}
    }
    \subfloat[]{
        \includegraphics[width=0.4\linewidth]{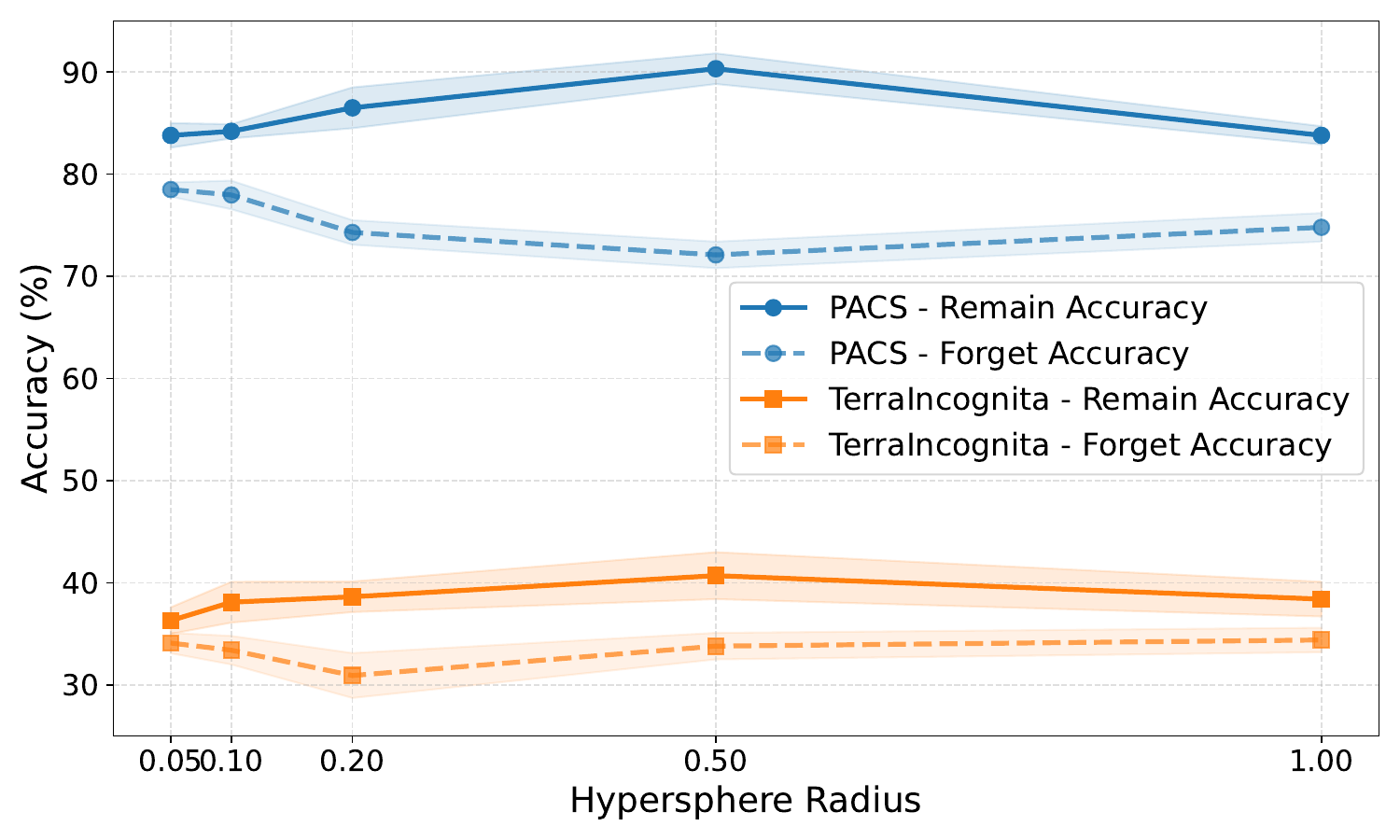}
        \label{fig:kappa}
    }
    \caption{Effects of different global learning rates and hypersphere radii on FOUL performance. For both PACS and Terra Incognita, the model achieves the best remaining accuracy (RA) and forgetting accuracy (FA) when $\eta = 0.2$ and $\kappa = 0.5$.}
    \vspace{-2mm}
\end{figure}

\paragraph{Learning rate.} In Fig.~\ref{fig:lr}, we present the FOUL results obtained with varying learning rates~$\eta$, while fixing $\kappa = 0.5$. Each result is averaged over three independent runs. For both the PACS and TerraIncognita datasets, the gradient matching task performs best when $\eta = 0.2$, achieving the best RA and FA.

\paragraph{Sensitivity of gradient matching.} We evaluate the sensitivity of the gradient matching process to different values of $\kappa$ on two datasets. Each experiment is conducted three times, and the results are averaged. As shown in Fig.~\ref{fig:kappa}, FOUL is optimal when the hypersphere radius is set to $0.5$ on PACS and between $0.2$ and $0.5$ on TerraIncognita, where the gap between RA and FA is the largest.

\paragraph{Disentanglement Efficiency.} To consider the disentanglement efficiency towards the unlearning phase, we deploy the ablation test according to different causal/non-causal representations dimensionality. We measure the RA and FA according to each cases and evaluate. The results is given in Figure~\ref{fig:vary-dim}.
% Given the disentanglement efficiency, we can achieve the computational efficiency on the unlearning phase due to the number of model parameters being frozen at the learning phase. The results is given in Table~\ref{tab:cost_analysis}.

\begin{figure}[!h]
    \centering
    \includegraphics[width=0.4\linewidth]{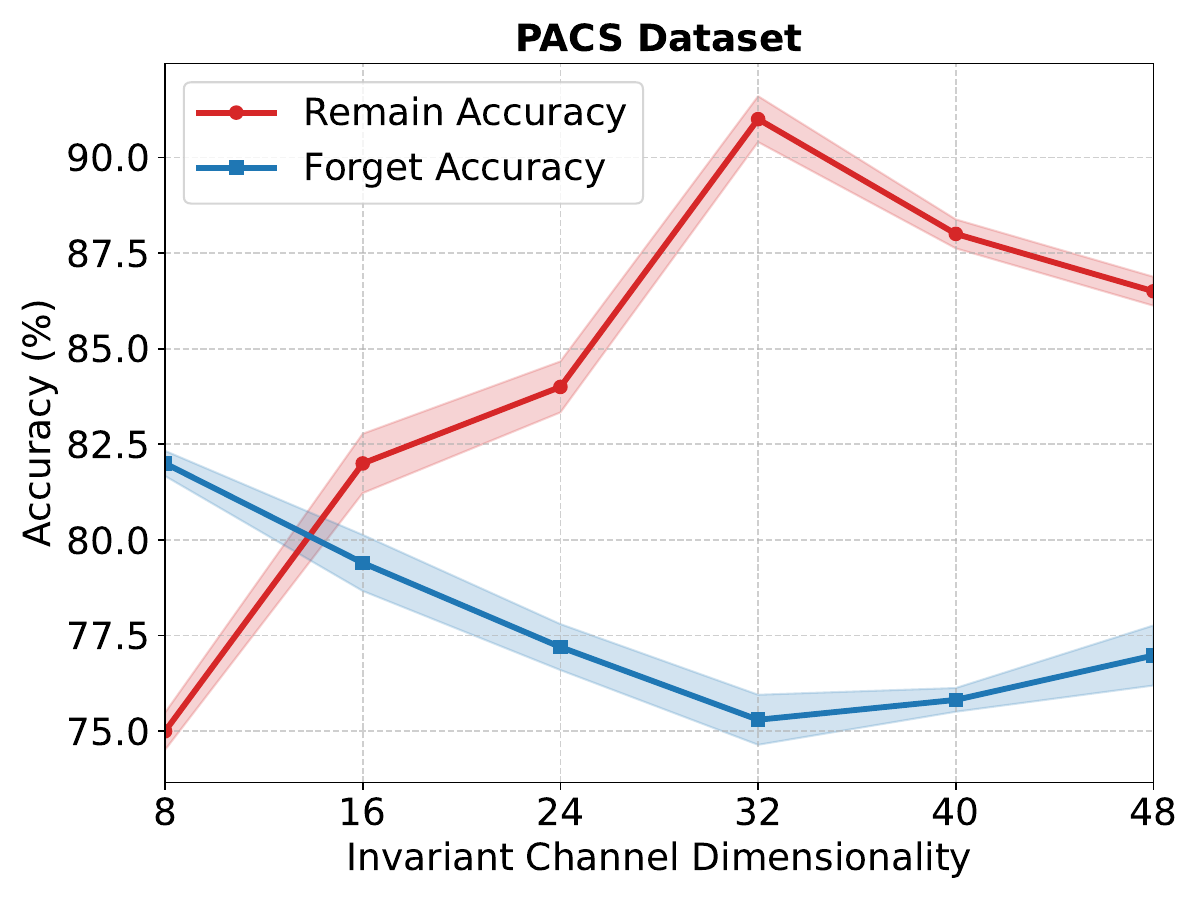}
    \includegraphics[width=0.4\linewidth]{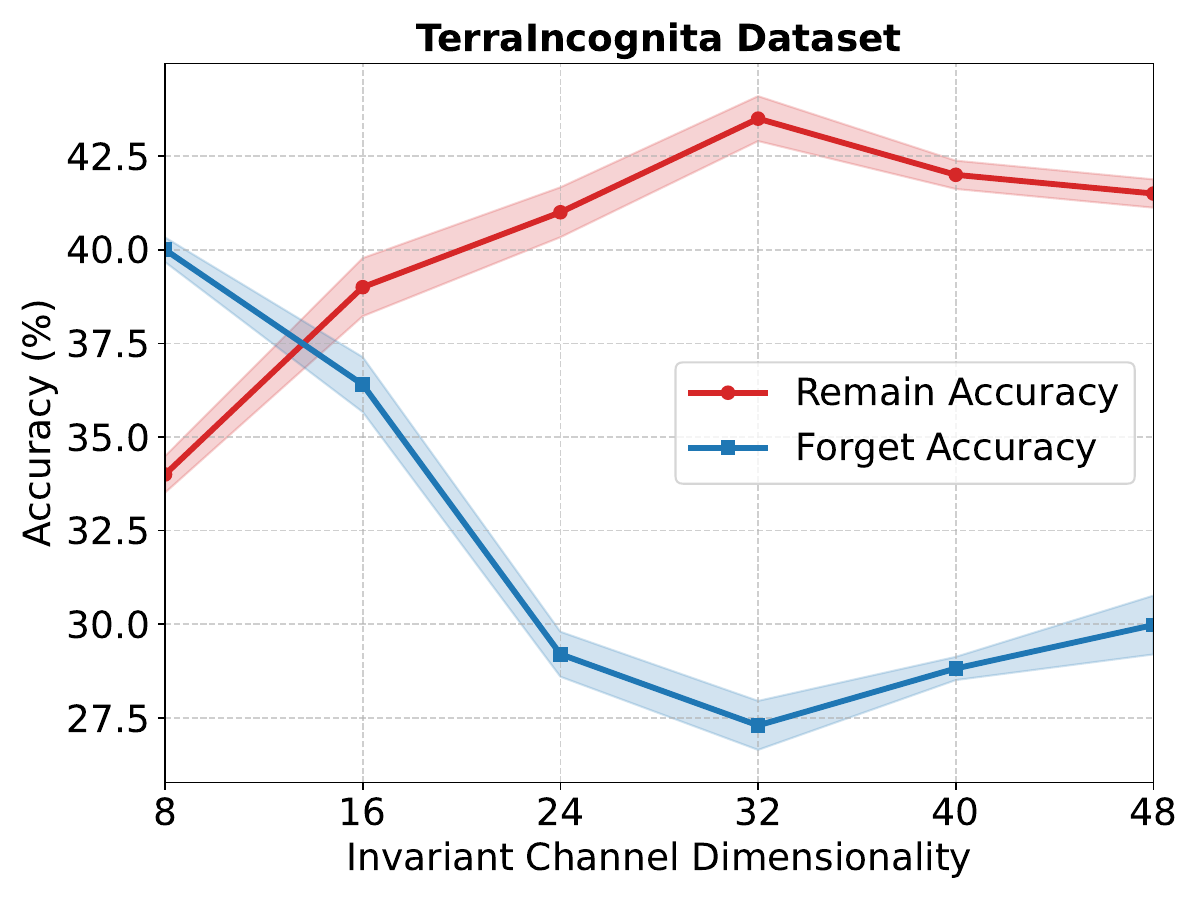}
    \caption{Ablation test on different causal dimensionality of the model disentanglement.}
    \vspace{-2mm}
    \label{fig:vary-dim}
\end{figure}

\section{Discussion} FOUL offers an efficient unlearning mechanism by directly aggregating information from both forget and retain clients on the server. Moreover, restricting unlearning to a sub-network significantly reduces computational cost. However, FOUL also has limitations. In particular, the gradient-alignment optimization performs best when the number of clients is moderate (e.g., fewer than $100$). When the client population becomes extremely large, the optimization may converge to suboptimal solutions. Addressing scalability and robustness in such large-client regimes is an important direction for future research.